\documentclass[9pt,twoside]{extarticle}
\usepackage[papersize={17cm,24cm},margin=22.5mm,top=14mm,headsep=5mm]{geometry}
\usepackage{amsmath, amssymb, amsthm, graphicx}
\usepackage[dvipsnames]{xcolor}
\usepackage{fancyhdr}
\fancypagestyle{plain}{
  \fancyhead[C,R]{}
  \fancyfoot{} 
  }

\usepackage[colorlinks=true,
  linkcolor=YellowOrange,
  citecolor=blue,
  bookmarks=False,
  hyperfootnotes=true]{hyperref}

\usepackage{times}
\usepackage{physics}
\usepackage{natbib}
\usepackage{mathtools}

\newcommand{\BEAS}{\begin{eqnarray*}}
\newcommand{\EEAS}{\end{eqnarray*}}
\newcommand{\BEA}{\begin{eqnarray}}
\newcommand{\EEA}{\end{eqnarray}}
\newcommand{\BA}{\begin{align}}
\newcommand{\EA}{\end{align}}
\newcommand{\BEQ}{\begin{equation}}
\newcommand{\EEQ}{\end{equation}}
\newcommand{\BIT}{\begin{itemize}}
\newcommand{\EIT}{\end{itemize}}
\newcommand{\BNUM}{\begin{enumerate}}
\newcommand{\ENUM}{\end{enumerate}}
\newcommand{\diag}{\mathop{\rm diag}}

\newcommand{\D}{\Delta}
\renewcommand{\O}{\mathcal{O}}
\newcommand{\M}{\mathcal{M}}
\newcommand{\X}{\mathcal{X}}
\newcommand{\F}{\mathcal{F}}
\newcommand{\g}{\mathfrak{g}}
\newcommand{\G}{\mathcal{G}_{d,k}}

\newcommand{\Hess}{\nabla^2}

\newcommand{\eps}{\varepsilon}
\newcommand{\tp}[2]{\Gamma_{#1}^{#2}}
\newcommand{\te}[2]{\Lambda_{#1}^{#2}}

\newcommand{\mysec}[1]{Section~\ref{sec:#1}}
\newcommand{\myapp}[1]{Appendix~\ref{sec:#1}}
\newcommand{\eq}[1]{Eq.~(\ref{eq:#1})}
\newcommand{\myfig}[1]{Figure~\ref{fig:#1}}

\newcommand{\Exp}{{\mathop { \rm Exp{}}}}

\newcommand{\idm}{I}
\newcommand{\rb}{\mathbb{R}}
\newcommand{\E}{\mathbb{E}}

\newtheorem{theorem}{Theorem}
\newtheorem{assumption}{Assumption}
\newtheorem{lemma}[theorem]{Lemma}
\newtheorem{proposition}[theorem]{Proposition}
\newtheorem{remark}{Remark}
\allowdisplaybreaks

\title{Averaging Stochastic Gradient Descent on Riemannian Manifolds}
\author{
Nilesh Tripuraneni\\
University of California, Berkeley\\
\texttt{nilesh\_tripuraneni@berkeley.edu}\\
\and
Nicolas Flammarion\\
University of California, Berkeley\\
\texttt{flammarion@berkeley.edu}\\
\and
\hspace{.3cm} Francis Bach\\
\hspace{.3cm} INRIA, Ecole Normale Sup\'erieure\\
\hspace{.3cm} PSL Research University\\
\hspace{.3cm} \texttt{francis.bach@inria.fr}\\
\and \hspace{.8cm}
Michael I. Jordan\\
\hspace{.8cm} University of California, Berkeley\\
\hspace{.8cm} \texttt{jordan@cs.berkeley.edu}\\
}
\date{}
\begin{document}

\maketitle

\begin{abstract}
We consider the minimization of a function defined on a Riemannian manifold $\mathcal{M}$ accessible only through unbiased estimates of its gradients. We develop a geometric framework to transform a sequence of slowly converging iterates generated from stochastic gradient descent (SGD) on $\mathcal{M}$ to an averaged iterate sequence with a robust and fast $O(1/n)$ convergence rate. We then present an application of our framework to geodesically-strongly-convex (and possibly Euclidean non-convex) problems.  Finally, we demonstrate how these ideas apply to the case of streaming $k$-PCA, where we show how to accelerate the slow rate of the randomized power method  (without requiring knowledge of the eigengap) into a robust algorithm achieving the optimal rate of convergence.
\end{abstract}

\section{Introduction}
We consider stochastic optimization of a (potentially non-convex) function $f$  defined on a Riemannian manifold $\M$, and accessible only through unbiased estimates of its gradients.
The framework is broad---encompassing fundamental problems such as principal components analysis (PCA) \citep{edelman1998geometry}, dictionary learning \citep{SunQinWri17}, low-rank matrix completion \citep{BouAbs11} and tensor factorization \citep{IshAbsVanDeL11}.

The classical setting of stochastic approximation in $\mathbb{R}^d$, first appearing in the work of
\citet{robbins1951stochastic}, has been thoroughly explored in both the optimization and machine learning communities. A key step in the development of this theory was the discovery of Polyak-Ruppert averaging---a technique in which the iterates are averaged along the optimization path. Such averaging provably reduces the impact of noise on the problem solution and improves convergence rates in certain important settings~\citep{Pol90, ruppert1988efficient}.

By contrast, the general setting of stochastic approximation  on Riemannian manifolds has been far less studied. There are many open questions regarding achievable rates and the possibility of accelerating these rates with techniques such as averaging.  The problems are twofold: it is not always clear how to extend a gradient-based algorithm to a setting which is missing the global vector-space structure of Euclidean space, and, equally importantly, classical analysis techniques often rely on the Euclidean structure and do not always carry over. In particular, Polyak-Ruppert averaging  relies critically on the Euclidean structure of the configuration space, both in its design and its analysis.  We therefore ask: \emph{Can the classical technique of Polyak-Ruppert iterate averaging be adapted to the Riemannian setting?}  Moreover, do the advantages of iterate averaging, in terms of rate and robustness, carry over from the Euclidean setting to the Riemannian setting?

Indeed, in the traditional setting of Euclidean stochastic optimization, averaged optimization algorithms not only improve convergence rates, but also have the advantage of adapting to the hardness of the problem \citep{moulines2011non}.  They provide a single, robust algorithm that achieves optimal rates with and without strong convexity, and they also achieve the statistically optimal asymptotic variance. In the presence of strong convexity, setting the step size proportional to $\gamma_n = \frac{1}{\mu n}$ is sufficient to achieve the optimal $O(\frac{1}{n})$ rate. However, as highlighted by \citet{NemJudLan08} and \citet{moulines2011non}, the convergence of such a scheme is highly sensitive to the choice of constant prefactor $C$ in the step size; an improper choice of $C$ can lead to an arbitrarily slow convergence \emph{rate}. Moreover, since $\mu$ is often never known, properly calibrating $C$ is often impossible (unless an explicit regularizer is added to the cost function, which adds an extra hyperparameter).

In this paper, we provide a practical iterate-averaging scheme that enhances the robustness and speed of stochastic, gradient-based optimization algorithms, applicable to a  wide range of Riemannian optimization problems---including those that are (Euclidean) non-convex. Principally, our framework extends the classical Polyak-Ruppert iterate-averaging scheme (and its inherent benefits) to the Riemannian setting.
Moreover, our results hold in the general stochastic approximation setting and do not rely on any finite-sum structure of the objective.

Our main contributions are:
\begin{itemize}
  \item The development of a geometric framework to transform a sequence of slowly converging iterates on $\M$, produced from SGD, to an iterate-averaged sequence
  with a robust, fast $O(\frac{1}{n})$ rate.
  \item A general formulation of geometric iterate averaging for a class of locally smooth and geodesically-strongly-convex optimization problems.
  \item An application of our framework to the (non-convex) problem of streaming PCA, where we show how to transform the slow rate of the randomized power method (with no knowledge of the unknown eigengap) into an algorithm that achieves the optimal rate of convergence and which empirically outperforms existing algorithms.
\end{itemize}

\subsection{Related Work}
\textbf{Stochastic Optimization:}
The literature on (Euclidean) stochastic optimization is vast, having been studied through the lens of machine learning \citep{bottou-98x, ShaShaSreSri09}, optimization \citep{NesVia08}, and stochastic approximation \citep{KusYin03}.
Polyak-Ruppert averaging first appeared in the works of \citet{Pol90} and \citet{ruppert1988efficient}; \citet{polyak1992acceleration} then provided asymptotic normality results for
the distribution of the averaged iterate sequence. \citet{moulines2011non} later generalized these results, providing non-asymptotic guarantees for the rate of convergence of the averaged iterates.
An important contribution of \citet{moulines2011non} was to present a unified analysis showing that iterate averaging coupled with sufficiently slow learning rates could achieve the
optimal convergence in all settings (i.e., with and without strong convexity).
\\
\textbf{Riemannian Optimization:}
Riemannian optimization has not been explored in the machine learning community until relatively recently. \citet{udriste1994convex} and
\citet{absil2009optimization} provide comprehensive background on the topic. Most existing work has primarily focused on providing asymptotic convergence guarantees for non-stochastic algorithms \citep[see, e.g.,][who analyze the convergence of Riemannian trust-region and Riemannian L-BFGS methods, respectively]{Absil2007,ring2012optimization}.

\citet{bonnabel2013stochastic}
provided the first asymptotic convergence
proof of stochastic gradient descent (SGD) on Riemannian manifolds while highlighting diverse applications of the Riemannian framework to problems such as PCA.
The first global complexity results for first-order Riemannian optimization, utilizing the notion of \emph{functional g-convexity}, were obtained
in the foundational work of \citet{zhang2016first}. The finite-sum, stochastic setting has been further investigated by \citet{zhang2016riemannian} and \citet{sato2017riemannian}, who developed Riemannian SVRG methods. However, the potential utility of Polyak-Ruppert averaging in the Riemannian setting has been unexplored.

\section{Results} \label{sec:results}
We consider the optimization of a function $f$ over a compact, connected subset $\X \subset \M$,
    \[
\min_{x \in \X \subset \M} f(x),
\] with access to a (noisy) first-order oracle $\{ \nabla f_n(x) \}_{n \geq 1}$. Given a sequence of iterates $\{x_n\}_{n\geq0}$  in~$\mathcal{M}$ produced from the first-order optimization of $f$,
\begin{align}
 x_{n} = R_{x_{n-1}} \left(-\gamma_n \nabla f_{n} \left(x_{n-1} \right)\right), \label{eq:grad_desc}
\end{align}
that are converging to a \emph{strict} local minimum of $f$, denoted by $x_\star$, we consider (and analyze the convergence of) a streaming average of iterates:
\begin{align}
 \tilde{x}_{n} = R_{\tilde{x}_{n-1}} \left(\frac{1}{n} R_{\tilde{x}_{n-1}}^{-1}\left(x_{n}\right)\right). \label{eq:ave_grad_desc}
\end{align}
Here we use $R_x$ to denote a retraction mapping (defined formally in \mysec{background}), which provides a natural means of moving along a vector (such as the gradient) while restricting movement to the manifold.
As an example, when $\M = \mathbb{R}^d$  we can take $R_x$ as vector addition by $x$. In this setting, \eq{grad_desc} reduces to the standard gradient update $x_n = x_{n-1} - \gamma_n \nabla f_n(x_{n-1})$ and \eq{ave_grad_desc} reduces to the ordinary average $\tilde{x}_n = \tilde{x}_{n-1} + \frac{1}{n}(x_{n} - \tilde{x}_{n-1})$.
In the update in \eq{grad_desc}, we will always consider step-size sequences of the form $\gamma_n = \frac{C}{n^\alpha}$ for $C>0$ and $\alpha \in \left(\frac{1}{2}, 1\right)$, which satisfy the usual stochastic approximation step-size rules $\sum_{i=1}^\infty \gamma_i=\infty$ and $\sum_{i=1}^\infty \gamma_i^2<\infty$ \citep[see, e.g.,][]{BenPriMet90}.

Intuitively, our main result
states that if the iterates $x_n$ converge to $x_\star$
at a slow $O(\gamma_n)$ rate, their streaming Riemannian average will converge to $x_\star$ at the the optimal $O(\frac{1}{n})$ rate. This result requires  several technical assumptions, which are standard generalizations of those appearing in the Riemannian optimization and stochastic approximation literatures (detailed in \mysec{assumptions}). The critical assumption we make is that all iterates remain bounded in $\X$---where the manifold behaves well and the algorithm is well-defined (Assumption~\ref{assump:manifold}).
The notion of slow convergence to an optimum is formalized in the following assumption:
\begin{assumption}\label{assump:slowrate}
  If $\Delta_n = R_{x_\star}^{-1}(x_n)$ for a sequence of iterates evolving in \eq{grad_desc}, then
  \[ \E[\Vert \Delta_n \Vert^2] = O(\gamma_n). \]
\end{assumption}
Assumption~\ref{assump:slowrate} can be verified in a variety of optimization problems, and
we provide such examples in \mysec{application}. As $x_\star$ is unknown, $\Delta_n$ is not computable but is primarily a tool for our analysis. Importantly, $\Delta_n$ is a tangent vector in $T_{x_\star} \M$. Note also that the norm $\Vert \Delta_n \Vert$  is locally equivalent to the geodesic distance $d(x_n,x_\star)$ on $\M$ (see \mysec{background}).

We use $\Sigma$ to denote the covariance of the noisy gradients at the optima $x_\star$.
Formally, our main convergence result regarding Polyak-Ruppert averaging in the manifold setting is as follows (where Assumptions~\ref{assump:manifold} through \ref{assump:noiseLip} will be presented later):
\begin{theorem} \label{thm:main}
  Let Assumptions \ref{assump:slowrate}, \ref{assump:manifold}, \ref{assump:strongconvpoint},
  \ref{assump:HessianLip}, \ref{assump:noiseunbiased},
  and  \ref{assump:noiseLip}
  hold for the iterates evolving according to \eq{grad_desc} and \eq{ave_grad_desc}.
    Then $\tilde{\Delta}_n = R_{x_{\star}}^{-1}(\tilde{x}_n)$ satisfies:
  \begin{align}
   \sqrt{n} \tilde{\Delta}_n \overset{D}{\to} \mathcal{N}(0, \Hess f(x_\star)^{-1} \Sigma \Hess f(x_\star)^{-1}). \notag
  \end{align}
  If we additionally assume a bound on the fourth moment of the iterates---of the form $\mathbb{E}[\Vert \Delta_n \Vert^4] = O(\gamma_n^2)$---then a
  non-asymptotic result holds:
  \begin{align}
    \mathbb{E}[\Vert \tilde{\Delta}_n \Vert^2] \leq \frac{1}{n} \tr[\Hess f(x_\star)^{-1} \Sigma \Hess f(x_\star)^{-1}] + O(n^{-2\alpha}) + O(n^{\alpha-2}). \notag
  \end{align}
  \end{theorem}
We make several remarks regarding this theorem:
\begin{itemize}
  \item The asymptotic result in Theorem \ref{thm:main} is a generalization of the classical asymptotic result of \citet{polyak1992acceleration}. In particular, the leading term has variance $O(\frac{1}{n})$ \textit{independently} of the step-size choice $\gamma_n$. In the presence of strong convexity, SGD can achieve the $O(\frac{1}{n})$ rate with a carefully chosen step size, $\gamma_n = \frac{C}{\mu n}$ (for $C=1$). However, the result is fragile: too small a value of $C$ can lead to an arbitrarily slow convergence rate, while too large a $C$ can lead to an ``exploding,'' non-convergent sequence \citep{NemJudLan08}. In practice determining $\mu$ is often as difficult as the problem itself.
  \item Theorem \ref{thm:main} implies that the distance  (measured in $T_{x_\star} \M$) of the streaming average $\tilde{x}_n$ to the optimum,
  asymptotically saturates the Cramer-Rao bound on the manifold $\M$ \citep{smi05, Bou13}---asymptotically achieving the statistically optimal covariance\footnote{Note the estimator $\tilde{\Delta}_n$ is only asymptotically unbiased, and hence the Cramer-Rao bound is only meaningful in the asymptotic limit. However, this result can also be understood as saturating the H\`{a}jek-Le Cam local asymptotic minimax lower bound \citep[Ch. 8]{van1998asymptotic}.}. SGD, even with the carefully calibrated step-size choice of $\gamma_n = \frac{1}{\mu n}$, does not achieve this optimal asymptotic variance \citep{NevHas73}.
\end{itemize}
We exhibit two applications of this general result in \mysec{application}. Next, we introduce the relevant background and assumptions that are necessary to prove our theorem.

\section{Preliminaries}\label{sec:background}
We recall some important concepts from Riemannian geometry. \citet{do2016differential}  provides more a thorough review, with \citet{absil2009optimization} providing a perspective particularly relevant for Riemannian optimization.

As a base space we consider a Riemannian manifold $(\M, \g)$---a smooth manifold equipped with a Riemannian metric $\g$ containing a compact, connected subset $\X$. At all $x \in \M$, the metric $\g$ induces a natural inner product on the tangent space $T_{x} \M$, denoted by $\langle\cdot,\cdot\rangle$---this inner product induces a norm on each tangent space denoted by $\Vert \cdot \Vert$. The metric $\g$ also provides a means of measuring the length of a parametrized curve from $\mathbb{R}$ to the manifold; a \emph{geodesic} is a constant speed curve $\gamma : [0,1] \to \M$ that is locally distance-minimizing with respect to the distance~$d$ induced by $\g$.

When considering functions $f : \M \to \mathbb{R}$ we will use $\nabla f(x) \in T_{x} M$ to denote the \emph{Riemannian gradient} of $f$ at $x \in \M$, and $\Hess f(x) : T_{x} M \to T_{x} M$, the \emph{Riemannian Hessian} of $f$ at $x$. When considering functions between manifolds $F : \M \to \M$, we will use $D F(x) : T_{x} \M \to T_{F(x)} \M$ to denote the \emph{differential} of the mapping at $x$ (its linearization) 
\citep[see][for more formal definitions of these objects]{absil2009optimization}.
 \begin{figure}[!ht]
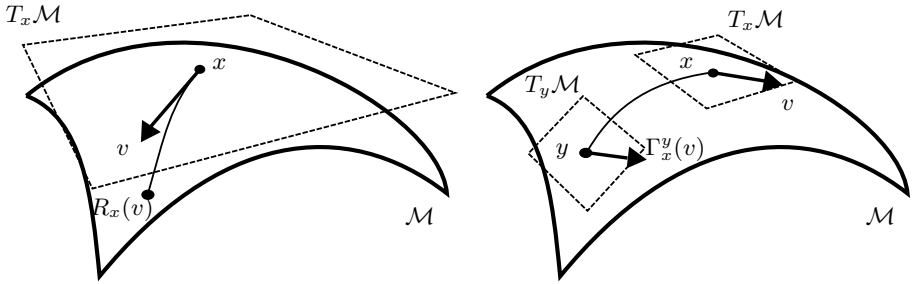

 \vspace{-.5cm}
\centering
\begin{minipage}[c]{.48\linewidth}
 \def\svgwidth{2.5in}
\input{Figs/retraction.tex}
   \end{minipage}
   \begin{minipage}[c]{.48\linewidth}
 \def\svgwidth{2.5in}
\input{Figs/transport.tex}
  \end{minipage}
  \vspace{-.5cm}
\caption{Left: a tangent vector $v$ in the tangent space of a point $x$ and the corresponding retraction generating a curve pointing in the ``direction'' of the tangent vector $v$. Right:
The parallel transport of a different $v$ along the same path.}
\label{fig:manifold}
\end{figure}
The \emph{exponential map} $\Exp_{x}(v) : T_{x} \M \to \M$ maps $v \in T_{x} \M$ to $y \in \M$ such that there is a geodesic with $\gamma(0)=x$, $\gamma(1)=y$, and $\frac{d}{dt}\gamma(0)=v$; although it may not be defined on the whole tangent space. If there is a unique geodesic connecting $x, y \in \X$, the exponential map will have a well-defined inverse $\Exp_{x}^{-1}(y) : \M \to T_{x}\M$, such that the length of the connecting geodesic is $d(x,y)=\Vert \Exp_{x}^{-1}(y) \Vert$. We also use $R_x : T_{x} \M \to \M$ and $R_x^{-1} : \M \to T_{x}\M$ to denote a \emph{retraction mapping} and its inverse (when well defined), which is an approximation to the exponential map (and its inverse). $R_x$ is often computationally cheaper to compute then the entire exponential map $\Exp_x$. Formally, the map $R_x$ is defined as a first-order retraction if $R_x(0) = x$ and $D R_x(0)= id_{T_{x} \M}$---so locally $R_x(\xi)$ must move in the ``direction'' of $\xi$. The map $R_x$ is a second-order retraction if $R_x$ also satisfies $\frac{D^2}{dt^2} R_x(t \xi)|_{0} = 0$ for all $\xi \in T_{x}\M$, where $\frac{D^2}{dt^2}\gamma = \frac{D}{dt} \dot{\gamma}$ denotes the acceleration vector field \citep[][Sec.~5.4]{absil2009optimization}. This condition ensures $R_x$ satisfies a ``zero-acceleration'' initial condition.
Note that locally, for $x$ close to $y$, the retraction satisfies $\norm{R^{-1}_{x}(y)} = d(x, y) + o(d(x, y))$.

If we consider the example where the manifold is a sphere (i.e., $\M = S^{d-1}$ with the round metric $\g$), the exponential map along a vector will generate a curve that is a great circle on the underlying sphere. A nontrivial example of a retraction $R$ on the sphere can be defined by first moving along the tangent vector in the embedded Euclidean space $\mathbb{R}^d$, and then projecting this point to the closest point on the sphere.

We further define the \emph{parallel translation} $\tp{x}{y} : T_x \mathcal{M} \to T_y \mathcal{M}$ as the map
transporting a vector $v \in T_x \M$ to $\tp{x}{y} v$, along a path $R_{x}(\xi)$ connecting $x$ to $y = R_{x}(\xi)$, such that the vector stays ``constant" by satisfying a zero-acceleration condition. This is illustrated in \myfig{manifold}. The map $\tp{x}{y}$ is an isometry. We also consider, a different vector transport map $\te{x}{y}:T_x \mathcal{M} \to T_y \mathcal{M}$ which is the differential  $DR_{x}(R_x^{-1}(y))$ of the retraction $R$ \citep[][Sec.~8.1]{absil2009optimization}.

Following \citet{huang2015broyden}, we will call a function $f$ on $\X$ \emph{retraction convex} on $\X$ (with respect to $R$) if for all $x \in \X$ and all $\eta \in T_{x}\M$ satisfying $\norm{\eta}=1$, $t \mapsto f(R_{x}(t \eta))$ is convex for all $t$ such that $R_{x}(t \eta) \in \X$; similarly $f$ is \emph{retraction strongly convex} on $\X$ if $t \mapsto f(R_{x}(t \eta))$ is strongly convex under the same conditions. If $R_x$ is the exponential map, this reduces to the definition of geodesic convexity \citep[see the work of][for further details]{zhang2016first}.

\section{Assumptions} \label{sec:assumptions}
We introduce several assumptions on the manifold $\M$, function $f$, and the noise process $\{\nabla f_n\}_{n\geq1}$ that will be relevant throughout the paper.
\subsection{Assumptions on $\M$}
First, we assume the iterates of the algorithm in \eq{grad_desc} and \eq{ave_grad_desc}
remain in $\mathcal X$ where the manifold ``behaves well.'' Formally,
\begin{assumption} \label{assump:manifold}
  For a sequence of iterates $\{ x_n \}_{n \geq 0}$ defined in \eq{grad_desc}, there exists a compact, connected subset $\mathcal X$ such that $x_n \in \mathcal X$ for all $n \geq 0$, and $x_\star \in \mathcal X$. Furthermore, $\mathcal X$ is totally retractive (with respect  to the retraction $R$) and the function $x\mapsto \Vert \frac{1}{2} R_y^{-1}(x)\Vert^2$ is retraction strongly convex on $\X$ for all $y\in\X$. Also, $R$ is a second-order retraction at $x_\star$.
\end{assumption}
Assumption~\ref{assump:manifold} is restrictive, but standard in prior work on stochastic approximation on manifolds \cite[e.g.,][]{bonnabel2013stochastic, zhang2016riemannian, sato2017riemannian}.
As further detailed by \citet{huang2015broyden}, a totally retractive neighborhood $\X$ is such that
for all $x \in \X$ there exists $r>0$ such that $\X \subset R_{x} (\mathbb{B}_{r}(0))$
where $R_x$ is a diffeomorphism on $\mathbb{B}_{r}(0)$. A totally retractive neighborhood is analogous to the concept of a totally normal neighborhood \citep[see, e.g.,][Chap. 3, Sec. 3]{do2016differential}.
Principally, Assumption~\ref{assump:manifold} ensures that the retraction map (and its respective inverse) are well-defined when applied to the iterates of our
algorithm.

If $\M$ is a Hadamard manifold, the exponential map (and its inverse) is defined everywhere on~$\M$, although this globally may not
be true for a retraction $R$. Similarly, if $\M$ is a compact manifold the first statement of Assumption~\ref{assump:manifold} is always satisfied. Moreover, in the case of the exponential map, $x \mapsto \frac{1}{2} \Vert \Exp_y^{-1}(x)\Vert^2$ is strongly convex in a ball around $y$
whose radius depends on the curvature, as explained by \citet{Afs11} and \citet[Ch. IV, Sec. 2 Lemma 2.9]{Sak96}. For our present purpose, we also assume the retraction $R$ agrees with the Riemannian exponential map up to second order near $x_\star$. This assumption, that $R$ is a second-order retraction, is fairly general and is satisfied by projection-like retraction maps on matrix manifolds \citep[see, e.g.,][]{AbsMal12}.

\subsection{Assumptions on $f$}
We now introduce some regularity assumptions on the function $f$ ensuring
sufficient differentiability and strong convexity at $x_\star$:
\begin{assumption} \label{assump:strongconvpoint}
The function $f$ is twice-continuously differentiable on $\X$. Further
the Hessian of the function $f$ at $x_\star$, $\Hess f(x_\star)$, satisfies,
for all $v \in T_{x_\star}\M$ and $\mu>0$,
\[
\langle v, \Hess f(x_\star)v\rangle \geq \mu \Vert v\Vert^2>0.
\]
\end{assumption}
Continuity of the Hessian  also ensures local retraction strong convexity in a neighborhood of $x_\star$ \citep[][Prop. 5.5.6]{absil2009optimization}. Moreover, since the function $f$ is twice-continuously differentiable on $\X$ its Hessian is Lipschitz on this compact set.
We formalize this as follows:
\begin{assumption}  \label{assump:HessianLip}
 There exists $M>0$ such that the Hessian of the function $f$, $\Hess f$, is $M$-Lipschitz  at $x_\star$.
That is, for all $y \in \X$ and $v \in T_{y}\M$,
\[
\Vert \tp{y}{x_\star}  \circ \Hess f(y) \circ \tp{x_\star}{y} -  \Hess f(x_\star)\Vert_{op} \leq M \Vert R_{x_\star}^{-1}(y) \Vert.
\]
\end{assumption}
Note that $\Vert R_{x_\star}^{-1}(y) \Vert$ is not necessarily symmetric under the exchange of $x_\star$ and $y$. This term could also be replaced with $d(x_\star, y)$, since these expressions will be locally equivalent
in a neighborhood of $x_\star$, but would come at the cost of a less transparent analysis.
\subsection{Assumptions on the noise} \label{sec:assumptions_noise}
We state several assumptions on the noise process that will be relevant throughout our discussion.
Let $(\F_n)_{n \geq 0}$ be an increasing sequence of sigma-fields. We will assume access to a sequence $\{\nabla f_n\}_{n \geq 1}$
of noisy estimates of the true gradient $\nabla f$ of the function $f$,
\begin{assumption} \label{assump:noiseunbiased}
The sequence of (random) vector fields $\{ \nabla f_n \}_{n\geq1} : \mathcal M \to T\mathcal M$ is $\F_n$-measurable, square-integrable and unbiased:
    \[
\forall x \in \X, \ \forall n \geq 1, \ \E[\nabla f_n(x)\vert \F_{n-1}]=\nabla f(x).
\]
\end{assumption}
This general framework subsumes two situations of interest.
\begin{itemize}
  \item
Statistical Learning (on Manifolds): minimizing a loss function $\ell: \mathcal M \times  \mathcal Z \to \mathbb R$ over $x \in \mathcal X$, given a sequence
of i.i.d. observations in $\mathcal Z$, with access only to noisy, unbiased estimates of the gradient $\nabla f_n=\nabla \ell(\cdot,z_n)$ \citep{AswBicTom11}.
\item
Stochastic Approximation (on Manifolds): minimizing a function $f(x)$ over $x \in \mathcal X$, with access only to the
(random) vector field $\nabla f(x) + \epsilon_n(x)$ at each iteration. Here the gradient vector field is perturbed
 by a square-integrable martingale-difference sequence (for all $x\in\mathcal M $, $\E[\epsilon_n(x) \vert \F_{n-1}]=0$) \citep{bonnabel2013stochastic}.
\end{itemize}

Lastly, we will assume the vector fields $\{\nabla f_n\}_{n \geq 1}$ are individually Lipschitz and have
bounded covariance at the optimum $x_\star$:
\begin{assumption} \label{assump:noiseLip}
 There exists $L > 0$ such that for all $x \in \mathcal X$ and $n\geq 1$, the vector field $\nabla f_n$ satisfies
\[ \E [\Vert  \tp{x}{x_\star} \nabla f_n(x)- \nabla f_n(x_\star)\Vert^2\vert  \F_{n-1}]\leq L^2 \ \Vert R_{x_\star}^{-1}(x) \Vert^2 ,
 \]
there exists $\tau>0$ such that $ \E [\Vert  \nabla f_{n}(x) \Vert^4\vert  \F_{n-1}]\leq \tau^4$ for all $x \in \X$, and a symmetric positive-definite matrix $\Sigma$ such that,
 \[
 \E[\nabla f_n(x_\star) \otimes \nabla f_n(x_\star) \vert\mathcal F_{n-1}] = \Sigma \text{ a.s.} \]
\end{assumption}
These are natural generalizations of standard assumptions in the optimization literature \citep{Fab68} to the setting of
Riemannian manifolds\footnote{Assuming bounded gradients does not contradict Assumption~\ref{assump:strongconvpoint}, since we are constrained to the compact set $\X$.}. Note that the assumption, \\ $\E[\nabla f_n(x_\star) \otimes \nabla f_n(x_\star) \vert\mathcal F_{n-1}] = \Sigma \text{ a.s.}$ could be slightly relaxed (as detailed in \myapp{conv_rates}),
but allows us to state our main result more cleanly.

\section{Proof Sketch} \label{sec:pfsketch}
We provide an overview of the arguments that comprise the proof of Theorem \ref{thm:main} (full details are deferred to \myapp{app_pfsketch}). We highlight three key steps. First, since we assume the iterates $x_n$
produced from SGD converge to within $\sim O(\sqrt{\gamma_n})$ of $x_\star$, we can perform a
Taylor expansion of the recursion in \eq{grad_desc}, to relate the points $x_n$ on the manifold $\M$ to vectors $\Delta_n$ in the tangent space $T_{x_\star}\M$. This
generates a (perturbed) linear recursion governing the evolution of the vectors $\Delta_n \in T_{x_\star} \M$.
Recall that as $x_\star$ is unknown, $\Delta_n$ is not accessible, but is primarily a tool for our analysis. Second, we can show a fast $O(\frac{1}{n})$ convergence rate for the averaged vectors $\bar{\Delta}_n \in T_{x_\star} \M$, using techniques from the Euclidean setting.
Finally, we once again use a local expansion of \eq{ave_grad_desc} to connect the averaged tangent vectors $\bar \Delta_n$ to the streaming, Riemannian average $\tilde \Delta_n$---transferring the fast rate for the inaccessible vector $\bar{\Delta}_n$ to the computable point $\tilde x_n$.
Throughout our analysis we extensively use Assumption~\ref{assump:manifold}, which restricts the iterates $x_n$ to the subset $\X$.

\subsection{From $\M$ to $T_{x_\star}\M$ } \label{sec:pfsketch1}
We begin by linearizing the progress of the SGD iterates $x_n$ in the tangent space of $x_\star$ by considering the evolution of $\Delta_n = R_{x_\star}^{-1}(x_n)$.
\begin{itemize}
  \item First, as the $\Delta_n$ are all defined in the vector space $T_{x_\star} \M$, Taylor's theorem applied  to $R_{x_\star}^{-1} \circ R_{x_n}:T_{x_n} \M \to T_{x_\star} \M$ along with \eq{grad_desc} allows us to conclude that
  \[
  \D_{n+1} = \D_n - \gamma_{n+1} [\te{x_\star}{x_n}]^{-1} (\nabla f_{n+1}(x_n)) + O(\gamma_{n+1}^2).
  \]
  See Lemma \ref{lem:tangent_rec} for more details.
  \item Second, we use the manifold version of Taylor's theorem and appropriate Lipschitz conditions on the gradient to further expand the gradient term $ \tp{x_n}{x_\star} \nabla f_{n+1}(x_n)$ as
  \[
 \tp{x_n}{x_\star} \nabla f_{n+1}(x_n)=\Hess f(x_\star)\Delta_n + \nabla f_{n+1}(x_\star)+ \xi_{n+1}+O(\Vert \Delta_n\Vert^2),
  \]
  where the noise term is controlled as $\E[\ \xi_{n+1}\vert\mathcal F_{n}]=0$, and $\E[\Vert \xi_{n+1}\Vert^2 \vert\mathcal F_{n}]=O( \Vert\Delta_n\Vert^2)$. See Lemma \ref{lem:tangent_rec_2} for more details.
  \item Finally, we argue that the operator $ [\te{x_\star}{x_n}]^{-1}\tp{x_\star}{x_n} : T_{x_\star}\M \to  T_{x_\star}\M$ is a local isometry up to second-order terms,
  \[
    [\te{x_\star}{x_n}]^{-1}\tp{x_\star}{x_n} = I + O(\norm{\Delta_n}^2),
  \]
  which crucially rests on the fact $R$ is a second-order retraction. See Lemma \ref{lem:tangent_rec_3} for more details.
  \item Assembling the aforementioned lemmas allows us to derive a (perturbed) linear recursion, governing the tangent vectors $\{ \Delta_n \}_{n \geq 0}$ as
  \begin{equation} \label{eq:final_proof_sketch}
    \D_{n+1} = \D_n - \gamma_{n+1} \Hess f(x_\star) \D_n  -\gamma_{n+1} \nabla f_{n+1}(x_\star)  -\gamma_{n+1}\xi_{n+1}   +  O(\norm{\D_n}^2\gamma_n + \gamma_n^2).
  \end{equation}
  See Theorem \ref{thm:linear} for more details.
\end{itemize}
\subsection{Averaging in $T_{x_\star} \M$} \label{sec:pfsketch2}
Our next step is to prove both asymptotic and non-asymptotic convergence rates for a general, perturbed linear recursion (resembling \eq{final_proof_sketch}) of the form,
\begin{align}
  \D_{n}=\D_{n-1} -\gamma_n \Hess f(x_\star) \D_{n-1}+ \gamma_n (\eps_n+\xi_{n}+e_{n}),\label{eq:rec_with_error}
\end{align}
under appropriate assumptions on the error $\{ e_n \}_{n \geq 0}$ and noise $\{ \eps_n \}_{n \geq 0}$, $\{ \xi_n \}_{n \geq 0}$ sequences detailed in \myapp{conv_rates}. Under these assumptions we can derive an asymptotic rate for the average, $\bar{\Delta}_n = \frac{1}{n}\sum_{i=1}^{n} \Delta_i$, under a first-moment condition on $e_n$:
  \[
  \sqrt n \bar{\Delta}_n  \overset{D}{\to} \mathcal N (0,  \Hess f(x_\star)^{-1}\Sigma \Hess f(x_\star)^{-1}),
  \]
  and, under a slightly stronger second-moment condition on $e_n$ we have:
  \[
    \mathbb{E}[\Vert \bar{\Delta}_n \Vert ^2] \leq \frac{1}{n} \tr [\Hess f(x_\star)^{-1} \Sigma \Hess f(x_\star)^{-1}] +  O(n^{-2\alpha}) + O(n^{\alpha-2}),
  \]
  where $\Sigma$ denotes the asymptotic covariance of the noise $\eps_n$. The proof techniques are similar to those of \citet{polyak1992acceleration} and \citet{moulines2011non} so we do not detail them here. See Theorems \ref{thm:asymp_ave} and \ref{thm:nonasymp_ave} for more details.
  Note that $\bar{\Delta}_n$ is \textit{not} computable, but does have an interesting interpretation as an upper bound on the Riemannian center-of-mass, $K_n = \arg \min_{x \in \M}\sum_{i=1}^{n} \norm{R_{x}^{-1}(x_i)}^2$, of a set of iterates $\{ x_n \}_{n \geq 0}$ in $\M$
   \citep[see \mysec{com} and][for more details]{Afs11}.

\subsection{From $T_{x_\star} \M$ back to $\M$} \label{sec:pfsketch3}
Using the previous arguments, we can conclude that the averaged vector $\bar{\Delta}_n$ obeys both asymptotic and non-asymptotic Polyak-Ruppert-type results. However, $\bar{\Delta}_n$
is \textit{not} computable. Rather, $\tilde{\Delta}_n = R_{x_\star}^{-1}(\tilde{x}_n)$ corresponds to the computable, Riemannian streaming average $\tilde{x}_n$ defined in \eq{ave_grad_desc}. In order to conclude our result, we argue that $\tilde{\Delta}_n = R_{x_\star}^{-1}(\tilde{x}_n)$ and $\bar{\Delta}_n$ are close up to $O(\gamma_n)$ terms. The argument proceeds in two steps:
\begin{itemize}
  \item Using the fact that $x \to \norm{R_{x_\star}^{-1}(x)}^2$ is retraction convex we can conclude that $\E[\norm{\Delta_n}^2] = O(\gamma_n)$ implies that  $\E[\Vert \tilde{\Delta}_n\Vert^2] = O(\gamma_n)$ as well. See Lemma \ref{lem:avg_iters} for more details.
  \item Then, we can locally expand \eq{ave_grad_desc} to find that,
  \[
    \tilde{\Delta}_{n+1} = \tilde{\Delta}_n + \frac{1}{n+1}(\Delta_{n+1}-\tilde{\Delta}_n)+\tilde{e}_n,
  \]
  where $\E[\Vert \tilde{e}_n\Vert] = O(\frac{\gamma_n}{n+1})$. Rearranging and summing this recursion shows that $\tilde{\Delta}_n = \bar{\Delta}_n+e_n$ for $\E[\Vert e_n\Vert] = O(\gamma_n)$, showing these terms are close. See Lemma \ref{lem:stream_avg_iters} for details.
\end{itemize}

\section{Applications} \label{sec:application}

We now introduce two applications of our Riemannian iterate-averaging framework.
\subsection{Application to Geodesically-Strongly-Convex Functions}
\label{sec:geostrong}
In this section, we assume that $f$ is globally geodesically convex over $\X$ and take $R \equiv \Exp$, which allows the derivation of global convergence rates. This function class encapsulates interesting problems such as the matrix Karcher mean problem
 \citep{bini2013computing} which is non-convex in Euclidean space but geodesically strongly convex with an appropriate choice of metric on $\M$.

\citet{zhang2016first} show for geodesically-convex $f$, that averaged SGD with step size $\gamma_n \propto \frac{1}{\sqrt{n}}$ achieves the slow $O\big(\frac{1}{\sqrt{n}}\big)$ convergence rate. If in addition, $f$ is geodesically strongly convex on $\X$, they obtain the fast $O(\frac{1}{n})$ rate. However, their result is not \textit{algorithmically robust}, requiring a delicate specification of the step size $\gamma_n \propto \frac{1}{\mu n}$, which is often practically impossible due to a lack of knowledge of $\mu$. Assuming smoothness of $f$, our iterate-averaging framework provides a means of obtaining a robust and global convergence rate. First, we make the following assumption:
\begin{assumption} \label{assump:strongconv}
The function $f$ is $\mu$-geodesically-strongly-convex on $\X$, for $\mu>0$, and the set $\X$ is geodesically convex.
\end{assumption}
Then using our main result in Theorem \ref{thm:main}, with $\gamma_n \propto \frac{1}{n^{\alpha}}$, we have:
\begin{proposition}   \label{prop:strongconvrate}
  Let Assumptions \ref{assump:manifold},
  \ref{assump:HessianLip},
  \ref{assump:noiseunbiased},
  \ref{assump:noiseLip}, and \ref{assump:strongconv}
  hold for the iterates evolving in \eq{grad_desc} and \eq{ave_grad_desc} and take the retraction $R$ to be the exponential map $\Exp$. Then,
  \[
    \mathbb{E}[\Vert \tilde{\Delta}_n \Vert^2] \leq \frac{1}{n} \tr[\Hess f(x_\star)^{-1} \Sigma \Hess f(x_\star)^{-1}] + O(n^{-2\alpha}) + O(n^{\alpha-2}).
  \]
\end{proposition}
We make several remarks.
\begin{itemize}
  \item In order to show the result, we first derive a
  slow rate of convergence for SGD,
 by arguing that $\E[d^2(x_{n},x_\star)]\leq \frac{2C  \zeta \upsilon^2}{\mu n^\alpha } + O( \exp(-c\mu n^{1-\alpha} ))$ and $
  \E[d^4(x_{n},x_\star)]\leq\frac{4C(3+\zeta)\zeta\upsilon^4}{\mu n^{2\alpha} } + O( \exp(-c\mu n^{1-\alpha} ))$ where $c, C> 0$ and $\zeta > 0$ is a geometry-dependent constant (see Proposition \ref{prop:mom2} for more details). The result follows by combining these results and Theorem \ref{thm:main}.

  \item As in Theorem \ref{thm:main} we also obtain convergence in law and the statistically optimal covariance.
  \item Importantly, taking the step size to be $\gamma_n \propto \frac{1}{\sqrt{n}}$ provides a single, robust algorithm achieving both the slow $O\big(\frac{1}{\sqrt{n}}\big)$ rate in the absence of strong convexity (by \citet{zhang2016first}) and the fast $O(\frac{1}{n})$ rate in the presence of strong convexity. Thus (Riemannian) averaged SGD automatically adapts to the strong-convexity in the problem without any prior knowledge of its existence (i.e., the value of $\mu$).
\end{itemize}
\subsection{Streaming Principal Component Analysis (PCA)} \label{sec:stream_pca}
The framework of geometric optimization is far-reaching, containing even (Euclidean) non-convex problems such as PCA. Recall the classical formulation of streaming $k$-PCA: we are given a stream of i.i.d.~symmetric positive-definite random matrices $H_n \in \rb^{d \times d}$ such that $\E H_n\!=\!H$, with eigenvalues $\{ \lambda_i \}_{1 \leq i \leq d}$ sorted in decreasing order, and hope to approximate the subspace of the top~$k$ eigenvectors, $\{ v_i \}_{1 \leq i \leq k}$. Sharp convergence rates for streaming PCA (with $k\!=\!1$) were first obtained by \citet{jain2016streaming} and \citet{shamir16b} using the randomized power method.
\citet{shamir2016fast} and \citet{AllenLi2017-streampca}
  later extended this work to the more general streaming $k$-PCA setting. These results are powerful---particularly because they provide \textit{global} convergence guarantees.

  For streaming $k$-PCA, a similar dichotomy to the convex setting exists: in the absence of an eigengap ($\lambda_k=\lambda_{k+1}$) one can only attain the slow $O\big(\frac{1}{\sqrt{n}}\big)$ rate, while the fast $O(\frac{1}{n})$ rate is achievable when the eigengap is positive ($\lambda_k>\lambda_{k+1}$).
However, as before, a practically burdensome requirement of these fast $O (\frac{1}{n} )$, global-convergence guarantees is that the step sizes of their corresponding algorithms depend explicitly on the unknown eigengap\footnote{In this example, the eigengap is analogous to the strong-convexity parameter $\mu$.} of the matrix $H$.

 By viewing the
  $k$-PCA problem as minimizing the Rayleigh quotient, $f(X) = -\frac{1}{2} \tr [X^\top H X]$, over the Grassmann manifold, we show how to apply the Riemannian iterate-averaging framework developed here to derive a fast, \textit{robust} algorithm,
\begin{align}
  X_n = R_{X_{n-1}} \left(\gamma_n H_n X_{n-1}\right) \quad \text{ and } \quad  \tilde X_n = R_{\tilde X_{n-1}}\Big(\frac{1}{n} X_{n}X_{n}^\top \tilde X_{n-1}\Big), \label{eq:robust_oja}
\end{align}
for streaming $k$-PCA. Recall that the Grassmann manifold $\G$ is the set of the $k$-dimensional subspaces of a $d$-dimensional Euclidean space which we equip with the projection-like, second-order retraction map
$R_X(V)=(X+V)[(X+V)^\top(X+V)]^{-1/2}$.
Observe that the randomized power method update \citep{OjaKar85},
$
X_n = R_{X_{n-1}}\big( \gamma_n H_n  X_{n-1} \big)
$, in \eq{robust_oja},
is almost identical to the Riemannian SGD update,
$
X_n = R_{X_{n-1}}\big(\gamma_n(I-X_{n-1} X_{n-1}^{\top}) H_n X_{n-1}\big)
$, in \eq{grad_desc}.
The principal difference between both is that in the randomized power method, the Euclidean gradient is used instead of the Riemannian gradient. Similarly, the average $ \tilde X_n = R_{\tilde X_{n-1}}\big(\frac{1}{n} X_{n}X_{n}^\top \tilde X_{n-1}\big)$, considered in \eq{robust_oja}, closely resembles the (Riemannian) streaming average in \eq{ave_grad_desc} (see \myapp{alg_stream}).

In fact we can argue that the randomized power method, Riemannian SGD, and the classic Oja iteration (the linearization of the randomized power method in $\gamma_n$) are equivalent up to $O(\gamma_n^2)$ corrections (see Lemma \ref{lem:equiv_oja}). The average $ \tilde X_n = R_{\tilde X_{n-1}}\big(\frac{1}{n} X_{n}X_{n}^\top \tilde X_{n-1}\big)$ also admits the same linearization as the Riemannian streaming average up to $O(\gamma_n)$ corrections (see Lemma \ref{lem:average_oja}).

Using results from  \citet{shamir2016fast} and  \citet{AllenLi2017-streampca} we can then argue that the randomized power method iterates satisfy a slow rate of convergence under suitable conditions on their initialization. Hence, the present framework is applicable and we can use geometric iterate averaging to obtain a local, robust, accelerated convergence rate. In the following, we will use $\{e_j \}_{1\leq j \leq k}$ to denote the standard basis vectors in $\mathbb{R}^k$.
\begin{theorem} \label{thm:oja_main}
  Let Assumption~\ref{assump:manifold} hold for the set $\X = \{ X : \Vert X_\star^\top X\Vert_F^2 \geq k-\eta \}$, for some constant $0 < \eta <\frac{1}{4}$, where $X_\star$ minimizes $f(X)$ over the $k$-Grassmann manifold. Denote, $\tilde{H}_n =  H^{-1/2} H_n H^{-1/2}$, and the 4th-order tensor $C_{ii'jj'} = \E [(v_i^\top \tilde H_n v_j) (v_{i'}^\top \tilde H_n v_{j'})]$. Further assume that $\norm{H_n}_2 \leq 1$ a.s., and that $\lambda_k > \lambda_{k+1}$. Then if $X_n$ and $\tilde{X}_n$ evolve according to \eq{robust_oja}, there exists a positive-definite matrix $C$, such that $\tilde{\Delta}_n = R_{X_{\star}}^{-1}(\tilde{X}_n)$ satisfies:
  \begin{align}
   \sqrt{n} \tilde{\Delta}_n \overset{D}{\to}  \mathcal{N}(0, C) \quad \text{ with } \quad  C = \sum_{j'=1}^k\sum_{i'=k+1}^d \sum_{j=1}^k\sum_{i=k+1}^d C_{ii'jj'} \frac{\sqrt{\lambda_i \lambda_j} \cdot \sqrt{\lambda_{i'} \lambda_{j'}}}{(\lambda_{j}-\lambda_{i}) \cdot (\lambda_{j'}-\lambda_{i'})} (v_i e_j^\top) \otimes (v_{i'} e_{j'}^\top). \notag
  \end{align}
\end{theorem}
We  make the following observations:
\begin{itemize}
  \item If the 4th-order tensor satisfies\footnote{For example if $H_n = h_n h_n^\top$ for $h_n \sim \mathcal{N}(0, \Sigma)$ -- so $H_n$ is a rank-one stream of Gaussian random vectors -- this condition is satisfied. See the proof of Theorem \ref{thm:oja_main} for more details. } $C_{ii'jj'} = \kappa \delta_{ii'}\delta_{jj'}$ for constant $\kappa$, the aforementioned covariance structure simplifies to, 
  \begin{align}
    C = \kappa \sum_{j=1}^k\sum_{i=k+1}^d\frac{{\lambda_i \lambda_j}}{(\lambda_j-\lambda_i)^2} (v_i e_j^\top) \otimes (v_{i} e_{j}^\top). \notag
\end{align}
  This asymptotic variance matches the result
  of \citet{reiss2016non}, achieving the same statistical performance as the empirical risk minimizer and matching the lower bound
  of \citet{CaiMaWu13} obtained for the (Gaussian) spiked covariance model. 
  \item Empirically, even using a constant step size in \eq{robust_oja} appears to yield convergence in a variety of situations; however, we can see a numerical counterexample in \myapp{counter}. We leave it as an open problem to understand the convergence of the iterate-averaged, constant step-size algorithm in the case of Gaussian noise \citep{BouLac85}.
 \item Assumption~\ref{assump:manifold} could be relaxed using a martingale concentration result showing the iterates $X_n$ are restricted to $\X$ with high probability similar to the work of \citet{shamir16b} and \citet{AllenLi2017-streampca}.
\end{itemize}
Note that we could also derive an analogous result to Theorem \ref{thm:oja_main} for the (averaged) Riemannian SGD algorithm in \eq{grad_desc} and \eq{ave_grad_desc}. However, we prefer to present the algorithm in \eq{robust_oja} since it is simpler and directly averages the (commonly used) randomized power method.

 \section{Experiments} \label{sec:experiments}

Here, we illustrate the practical utility of our results on a synthetic, streaming $k$-PCA problem using the SGD algorithm defined in \eq{robust_oja}. We take $k=10$ and $d=50$. The stream $H_n \in \mathbb{R}^d$ is normally-distributed with a covariance matrix $H$ with random eigenvectors, and eigenvalues decaying as $1/i^\alpha+\beta$, for $i = 1, \dots, k$, and $1/(i-1)^\alpha$, for $i = i+1,\dots, d$  where $\alpha,\beta\geq0$. All results are averaged over ten repetitions.

 \paragraph{Robustness to Conditioning.}  In \myfig{synthetic} we consider two covariance matrices with different conditioning and we compare the behavior of SGD and averaged SGD for different step sizes (constant (cst), proportional to $1/\sqrt{n}$ (-1/2) and  $1/n$ (-1)). When the covariance matrix is well-conditioned, with a large eigengap (left plot), we see that SGD converges at a rate which depends on the step size whereas averaged SGD converges at a $O(1/n)$ rate independently of the step-size choice.  For poorly conditioned problems (right plot), the convergence rate deteriorates to $1/\sqrt{n}$ for non-averaged SGD with step size $1/\sqrt{n}$, and averaged SGD with both constant and $1/\sqrt{n}$ step sizes. The $1/n$ step size performs poorly with and without averaging.
 
  \begin{figure}[!ht]
 \vspace{-1cm}
\centering
\begin{minipage}[c]{.5\linewidth}
\includegraphics[width=\linewidth]{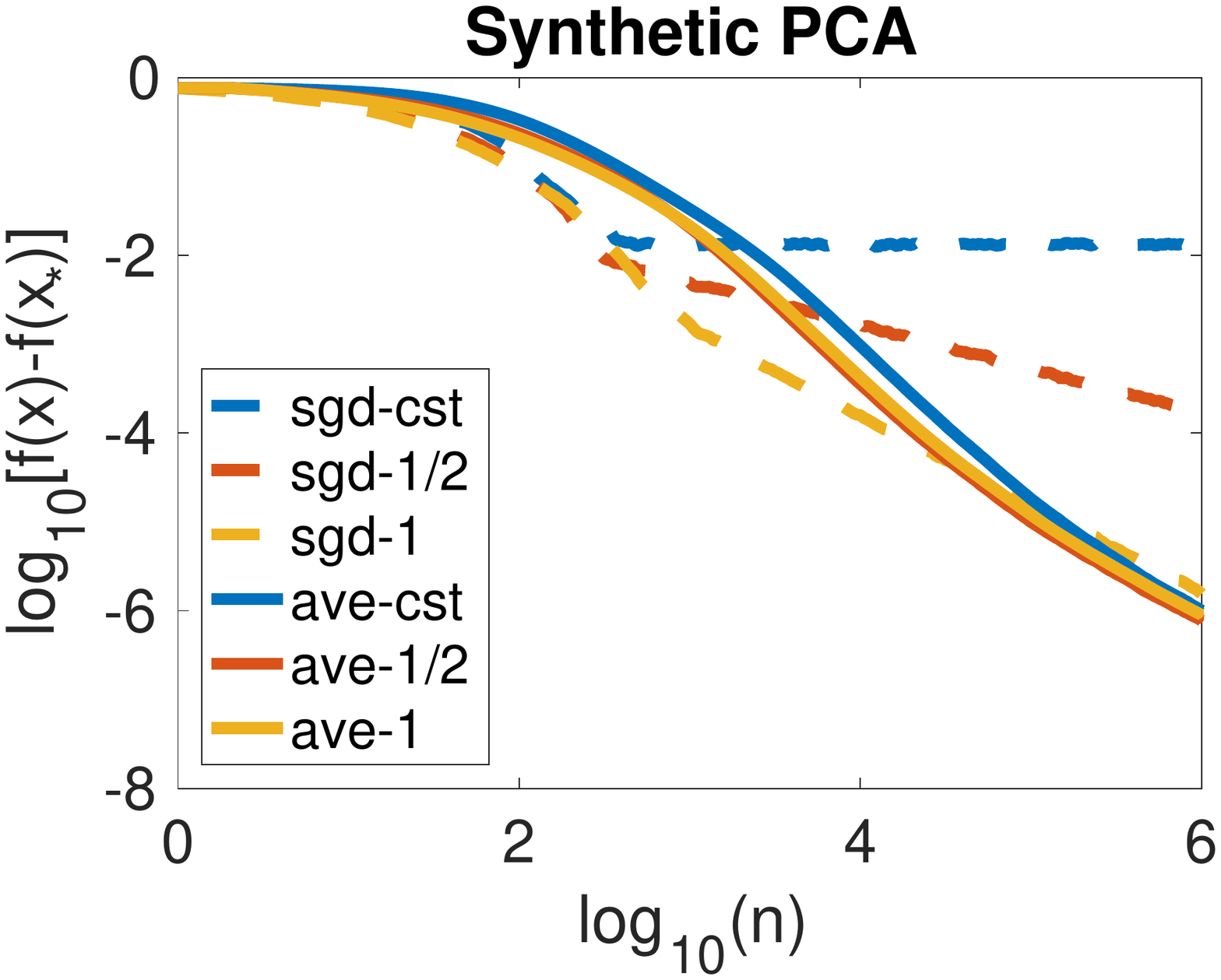}
   \end{minipage}
    \hspace*{-10pt}
   \begin{minipage}[c]{.5\linewidth}
\includegraphics[width=\linewidth]{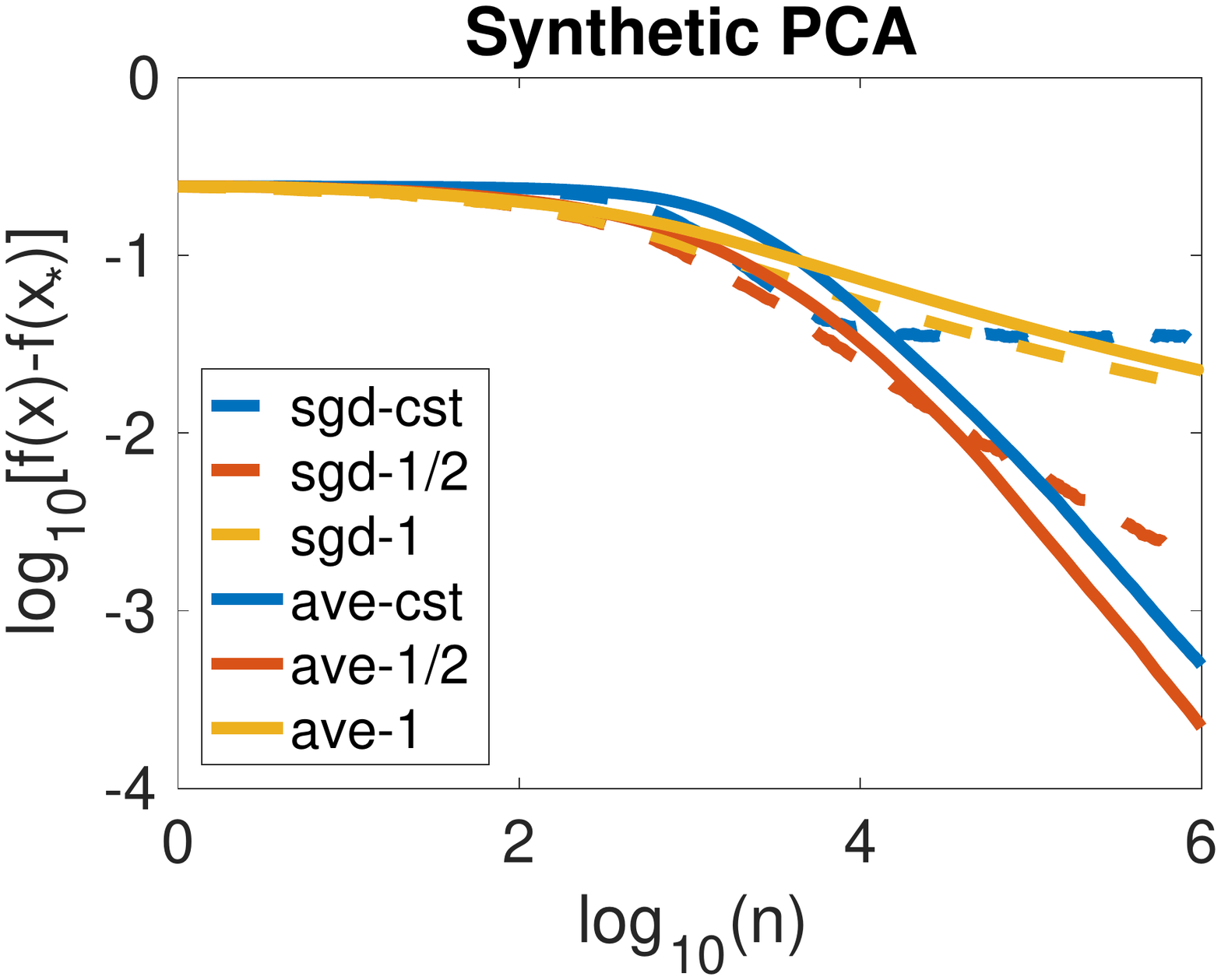}
   \end{minipage}
   \vspace{-2.0cm}
  \caption{Streaming PCA. Left: Well-conditioned problem. Right: Poorly-conditioned problem.}
     \label{fig:synthetic}
  \vspace{0cm}
\end{figure}

\paragraph{Robustness to Incorrect Step-Size.} In \myfig{syntheticrob} we consider a well-conditioned problem and compare the behavior of SGD and averaged SGD with step size proportional to $C/\sqrt n$ and $C/n$ to investigate the robustness to the choice of the constant $C$. For both algorithms we take three different constant prefactors $C/5$, $C$ and $5C$. For the step size proportional to $C/\sqrt n$ (left plot), both SGD and averaged SGD are robust to the choice of $C$. For SGD, the iterates eventually converge at a  $1/\sqrt{n}$ rate, with a constant offset proportional to $C$. However, averaged SGD enjoys the fast rate $1/n$ for all choices of $C$. For the step size proportional to $C/n$ (right plot), if $C$ is too small, the rate of convergence is extremely slow for SGD and averaged SGD.

 \begin{figure}[!ht]
 \vspace{-1cm}
\centering
\begin{minipage}[c]{.5\linewidth}
\includegraphics[width=\linewidth]{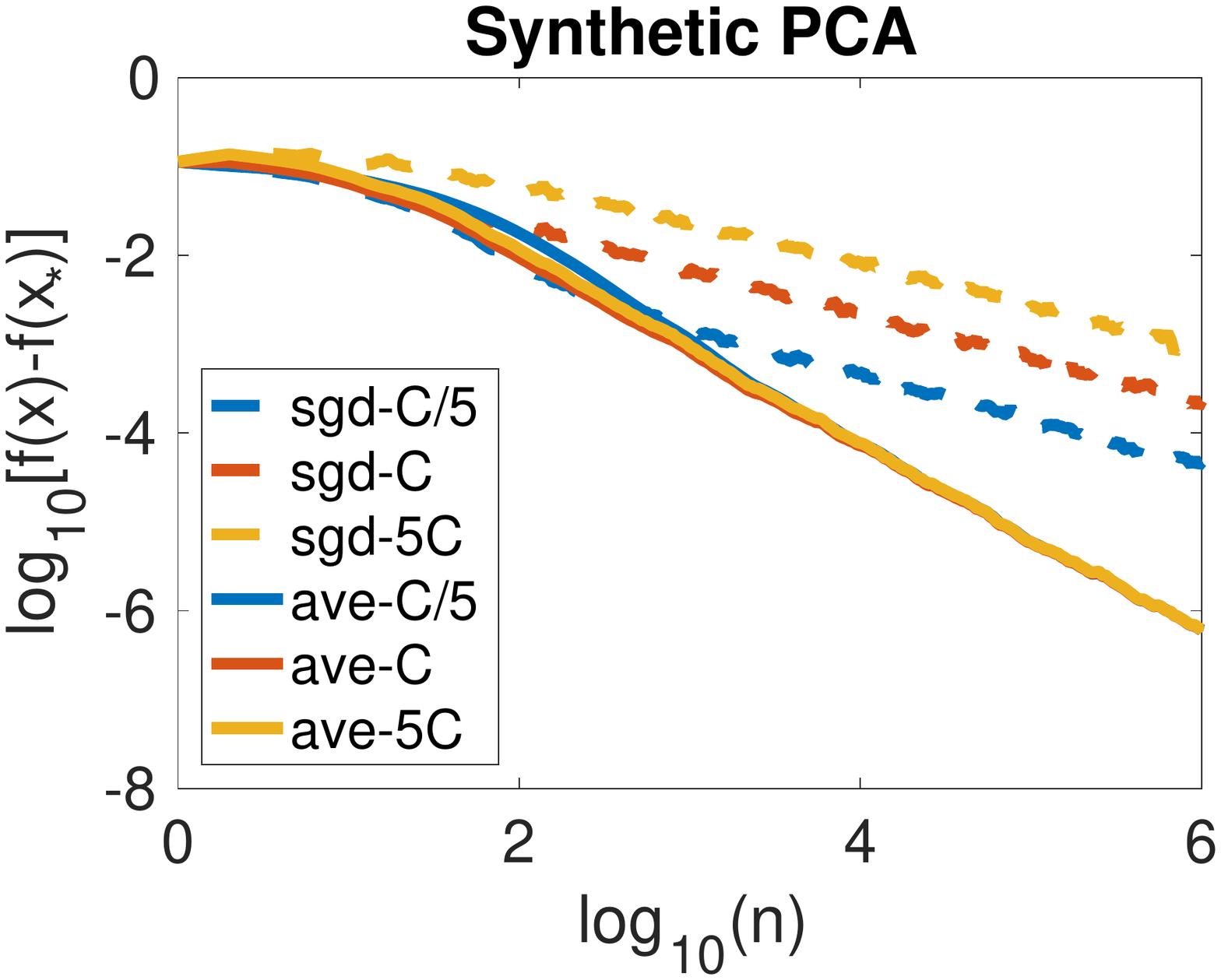}
   \end{minipage}
   \hspace*{-10pt}
   \begin{minipage}[c]{.5\linewidth}
\includegraphics[width=\linewidth]{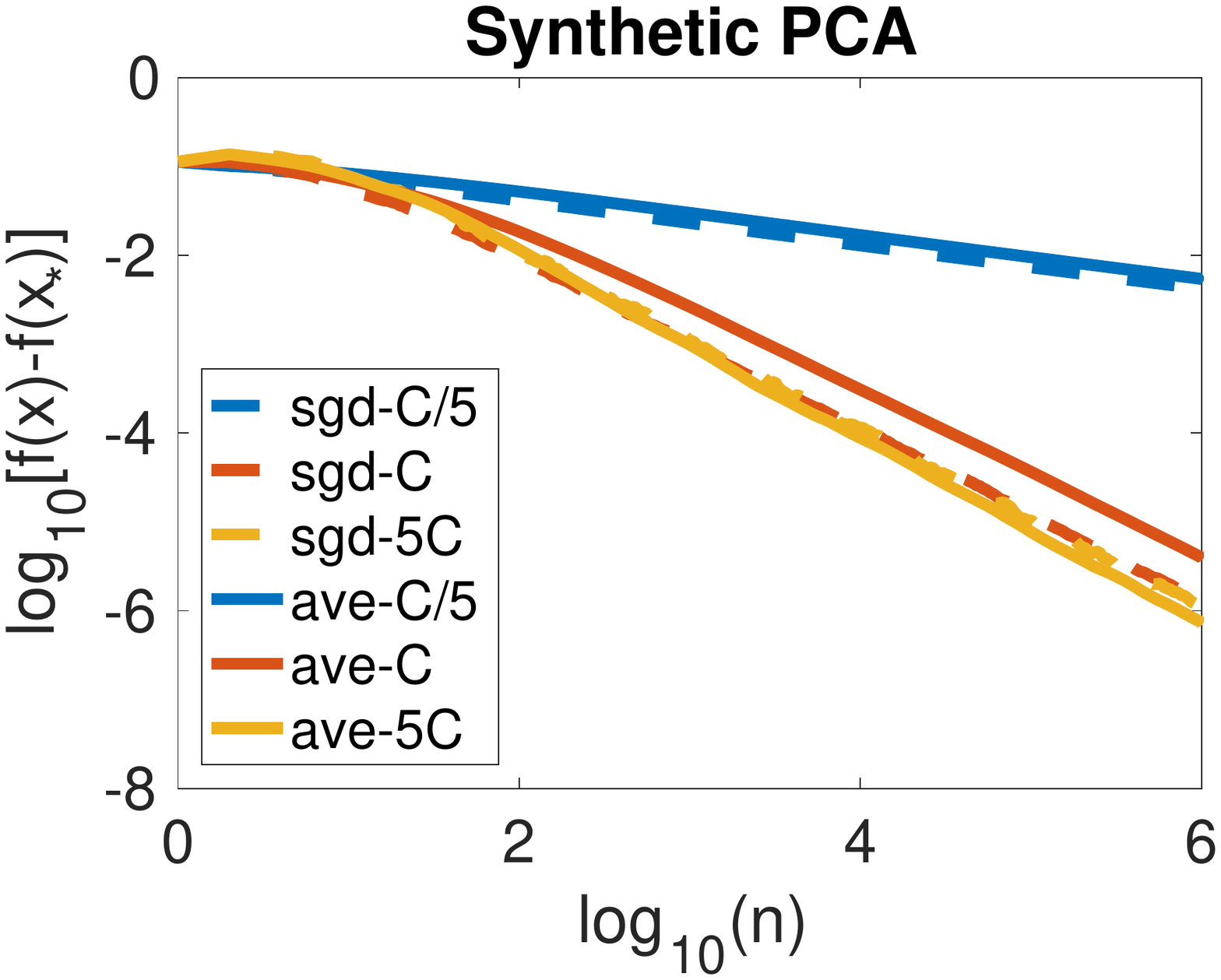}
   \end{minipage}
   \vspace{-2.0cm}
  \caption{Robustness to constant in step size. Left:
   step size proportional to $n^{\!-\!1/2}$. Right:
   step size proportional to $n^{\!-\!1}$.}
     \label{fig:syntheticrob}
\end{figure}
\section{Conclusions}
We have constructed and analyzed a geometric framework on Riemannian manifolds that generalizes the classical Polyak-Ruppert iterate-averaging scheme.  This framework is able to accelerate a sequence of slowly-converging iterates to an iterate-averaged sequence with a robust $O(\frac{1}{n})$ rate. We have also presented two applications, to the class of geodesically-strongly-convex optimization problems and to streaming $k$-PCA.
Note that our results only apply locally, requiring the iterates to be constrained to lie in a compact set $\X$. Considering a projected variant of our algorithm as in \citet{FlaBac17} is a promising direction for further research that may allow us to remove this restriction. Another interesting direction is to provide a global-convergence result for the iterate-averaged PCA algorithm presented here.

\subsection*{Acknowledgements}
The authors thank Nicolas Boumal and John Duchi for helpful discussions. Francis Bach acknowledges support from the European Research Council (grant SEQUOIA 724063), and Michael Jordan acknowledges support from the Mathematical Data Science program of the Office of Naval Research under grant number N00014-15-1-2670. 

\bibliographystyle{plainnat}
\bibliography{onlinepca}

\newpage
\onecolumn
\appendix
\section{Appendices}

In \myapp{main_proof} we provide the proof of Theorem \ref{thm:main}. In \myapp{app_pfsketch} we prove the relevant lemmas mirroring the proof sketch in \mysec{pfsketch}. In \myapp{appapp} we provide proofs of the results for the application discussed in \mysec{geostrong} about geodesically-strongly-convex optimization. \mysec{stream_pca_app} contains background and proofs of results discussed in \mysec{stream_pca} regarding streaming $k$-PCA. \mysec{counter} contains further experiments on synthetic PCA showing a counterexample to the convergence of averaged, constant step-size SGD mentioned in \mysec{experiments} in the main text.

Throughout this section we will denote a sequence of vectors $X_{n}$ to be $X_{n} = O(f_n)$, for scalar functions $f_n$,
if there exists a constant $C>0$, such that $\norm{X_{n+1}} \leq C f_n$ for all $n \geq 0$ almost surely.
\section{Proofs for \mysec{results}} \label{sec:main_proof}
Here we provide the proof of Theorem \ref{thm:main}. The first statement follows by combining Theorems \ref{thm:linear}, \ref{thm:asymp_ave}, Lemma \ref{lem:stream_avg_iters} and Slutsky's theorem. The second statement follows by using Theorems \ref{thm:linear}, \ref{thm:nonasymp_ave}, and Lemma \ref{lem:stream_avg_iters_4mom}.

\section{Proofs for \mysec{pfsketch}} \label{sec:app_pfsketch}
Here we detail the proofs results necessary to conclude our main result sketched in
\mysec{pfsketch}.

\subsection{Proofs in \mysec{pfsketch1}}
We begin with the proofs of the geometric lemmas detailed in \mysec{pfsketch1},
showing how to linearize the progress of the SGD iterates $x_n$ in the tangent space of $x_\star$ by considering the evolution of $\Delta_n = R_{x_\star}^{-1}(x_n)$. Note that since by Assumption~\ref{assump:manifold}, for all $n\geq0$, $x_n\in\X$, the vectors $\Delta_n$ all belong to the compact set $R_{x_\star}^{-1}(\X)$.


In the course of our argument it will be useful to consider the function $F_{x, y}(\eta_x) = R_y^{-1} \circ R_x(\eta_x) : T_x \M \to T_x \M$
 (which crucially
is a function defined on a vector space) and further $D R_x(\eta_x) : T_x \M \to T_{R_x(\eta_x)} \M$, the linearization of the retraction map. The first recursion we will study is that of $\Delta_{n+1} = F_{x_n, x_\star}(-\gamma_{n+1} \nabla f_{n+1}(x_n))$:
\begin{lemma}\label{lem:tangent_rec}
Let Assumption \ref{assump:manifold} hold. If $\Delta_n = R_{x_\star}^{-1}(x_n)$ for a sequence of iterates evolving as in  \eq{grad_desc}, then there exists a constant $C_{\text{manifold}}>0$ depending on $\X$ such that,
  \[
  \D_{n+1} = \D_n - \gamma_{n+1} [\te{x_\star}{x_n}]^{-1} (\nabla f_{n+1}(x_n)) + \gamma_{n+1} g_n,
  \]  where $\norm{g_n} \leq \gamma_{n+1} C_{\text{manifold}} \norm{\nabla f_{n+1}(x_n)}^2$.
\end{lemma}
\begin{proof}
  Using the chain rule for the differential of a mapping on a manifold and the first-order property of the retraction ($D R_x (0_x) = \idm_{T_x\M}$)
  we have that:
  \begin{multline*}
      D F_{x, y}(0_x) = D(R_y^{-1} \circ R_x)(0_x) = D R_{y}^{-1}(R_x(0_x)) \circ D R_x(0_x) \\= [D R_y(R_{y}^{-1}(R_x(0_x)))]^{-1} \circ \idm_{T_x \M} =    [DR_y(R_{y}^{-1}(x))]^{-1}=[\te{y}{x}]^{-1},
  \end{multline*}
  where the last line follows by the inverse function theorem on the manifold $\M$.
 Smoothness of the retraction  implies the Hessian of $F_{x, y}$ is uniformly bounded in norm on the compact set $F_{x, y}^{-1}(R_{x_\star}^{-1}(\X))$. We use $C_{\text{manifold}}$ to denote a bound on the operator norm of the Hessian of $F_{x, y}$ in this compact set. In the present situation,
  we have that $\Delta_{n+1} = F_{x_n, x_\star}(-\gamma_{n+1} \nabla f_{n+1}(x_n))$. Since $F_{x_n, x_\star}$ is a function defined on vector spaces the result follows using a Taylor expansion,
  $F_{x_n, x_\star}(0)=\Delta_n$, the previous statements regarding the differential of $F_{x_n, x_\star}$, and the uniform bounds on the second-order terms. In particular, the second-order term in the Taylor expansion is upper bounded as $\gamma_{n+1} C_{\text{manifold}} \norm{\nabla f_{n+1}(x_n)}^2$ so the bound on the error term $g_n$ follows.
\end{proof}
We now further develop this recursion by also considering an asymptotic expansion of the gradient term near the optima.
\begin{lemma}\label{lem:tangent_rec_2}
Let Assumptions \ref{assump:HessianLip}, \ref{assump:noiseunbiased},  and \ref{assump:noiseLip} hold.
If $\Delta_n = R_{x_{\star}}^{-1}(x_n)$ for a sequence of iterates evolving as in \eq{grad_desc},  then there exist sequences $\{ \tilde \xi_n \}_{n \geq 0}$ and $\{ \tilde e_n \}_{n \geq 0}$ such that
  \[
 \tp{x_n}{x_\star} \nabla f_{n+1}(x_n)=\Hess f(x_\star)\Delta_n + \nabla f_{n+1}(x_\star)+\tilde \xi_{n+1}+\tilde e_{n+1},
  \]
where $\E[\ \tilde \xi_{n+1}\vert\mathcal F_{n}]=0$, $\E[\Vert \tilde\xi_{n+1}\Vert^2 \vert\mathcal F_{n}]\leq 4 L \Vert \Delta_n\Vert^2$ and $\tilde e_{n+1}$ such that $\Vert\tilde e_{n+1}\Vert \leq \frac{M}{2}\Vert \Delta_n\Vert^2$.
\end{lemma}
\begin{proof}
We begin with the decomposition:
\BEAS
\Hess \!\! f(x_\star)\Delta_n&\!\!\!=\!\!\!& \tp{x_n}{x_\star} \nabla f(x_n) - \nabla f(x_\star) +[\Hess \!\!f(x_\star)\Delta_n-\tp{x_n}{x_\star} \nabla f(x_n) - \nabla f(x_\star) ] \\
&\!\!\!=\!\!\!&  \tp{x_n}{x_\star}\!\! \nabla f_{n+1}(x_n) - \nabla f_{n+1}(x_\star) +[\Hess \!\!f(x_\star)\Delta_n-\tp{x_n}{x_\star} \!\!\nabla f(x_n) \!-\! \nabla f(x_\star) ] \\
&&+  [\tp{x_n}{x_\star} \nabla f(x_n) - \nabla f(x_\star) - \tp{x_n}{x_\star} \nabla f_{n+1}(x_n) + \nabla f_{n+1}(x_\star)].
\EEAS
Under Assumption \ref{assump:HessianLip}, using the manifold version of Taylor's theorem (see \citet{absil2009optimization} Lemma 7.4.8) we have for $\tilde e_{n+1}=   \Hess f(x_\star) \Delta_n -\tp{x_n}{x_{\star}} \nabla f(x_n)$, that
\[
\Vert \tilde e_{n+1}\Vert \leq \frac{M}{2}\Vert \D_n\Vert^2.
\]
Denoting $\tilde \xi_{n+1}=  [\tp{x_n}{x_\star} \nabla f(x_n) - \nabla f(x_\star) - \tp{x_n}{x_\star} \nabla f_{n+1}(x_n) + \nabla f_{n+1}(x_\star)]$, Assumption \ref{assump:noiseunbiased} directly implies that $\E[\ \tilde \xi_{n+1}\vert\mathcal F_{n}]=0$. Finally, using Assumption \ref{assump:noiseLip} and the elementary inequality $2 \E[A \cdot B | \F_n] \leq \E[A^2 | \F_n] + \E[B^2 | \F_n]$ for square-integrable random variables $A, B$ shows that,
\BEAS
\E[\Vert \tilde\xi_{n+1}\Vert^2 \vert\mathcal F_{n}]
&\leq& 2\Vert \tp{x_n}{x_\star} \nabla f(x_n) - \nabla f(x_\star) \Vert^2 + 2 \E \left[\Vert \tp{x_n}{x_\star} \nabla f_{n+1}(x_n) - \nabla f_{n+1}(x_\star) \Vert^2 | \F_n \right] \\
&\leq &4L^2 \Vert \D_n \Vert^2.
\EEAS
\end{proof}
The last important step to conclude a linear recursion in $\Delta_n$ is to argue that the operator composition
$ [\te{x_\star}{x_n}]^{-1}\tp{x_\star}{x_n} : T_{x_\star}\M \to  T_{x_\star}\M$, is in fact an
isometry (up to 2nd-order terms) since $x_n$ is close to $x_\star$. The following argument crucially uses the fact that $R_{x_\star}$ is a second-order retraction.
\begin{lemma} \label{lem:tangent_rec_3}
 Let Assumption \ref{assump:manifold}. Let $\Delta_n = R_{x_{\star}}^{-1}(x_n)$ for a sequence  $\{ x_n \}_{n \geq 0}$ evolving as in \eq{grad_desc}. Then there exists a trilinear operator $K(\cdot,\cdot,\cdot)$ such that
  \[
    [\te{x_\star}{x_n}]^{-1}\tp{x_\star}{x_n} = I - K(\Delta_n,\Delta_n, \cdot)   + O(\norm{\Delta_n}^3).
  \]
\end{lemma}
As noted in the proof, when the exponential map is used as the retraction, the operator $K$ is precisely the Riemann curvature tensor $R_{x_\star}(\D_n, \cdot)\D_n$ (up to a constant prefactor).
\begin{proof}
  We derive a Taylor expansion for the operator composition $[\tp{x}{y}]^{-1}\te{x}{y}$ when $y$ is close to $x$. Consider the function $G(v)= [\tp{x}{R_x(v)}]^{-1}\te{x}{R_x(v)}: T_x\M\to L(T_x\M)$ where $L(T_x\M)$ denotes the set of linear maps on the vector space $T_x\M$. Now, recall that $\tp{x}{R_x(tv)}$ is precisely the parallel translation operator along the curve $\gamma(t)=R_y(tv)$. From Proposition 8.1.2 by \citet{absil2009optimization}, we have that
  \[
  \frac{d}{dt} G(tv)\vert_{t=0} =   \frac{d}{dt}  [\tp{x}{R_x(tv)}]^{-1}\te{x}{R_x(tv)}\vert_{t=0} = \nabla_{\dot \gamma(0)} D R_y,
  \]
  where $\nabla$ denotes the Levi-Civita connection (see also the proof of \citet[Lemma 7.4.7]{absil2009optimization} and \citet[Chapter 2, Exercise 2]{do2016differential}). Using the definition of the covariant derivative $\nabla_{v}$ along a vector $v$, and of the acceleration vector field $\dot{\gamma}$ \citep[Section 5.4]{absil2009optimization}
  we have that
  \[
   \nabla_{\dot \gamma(0)} D R_y= \frac{D}{dt} D R_y(\gamma(t))\vert_{t=0} = \frac{D^2}{dt^2} R_y(tv)\vert_{t=0}=0,
  \]
  since $R$ is a second-order retraction. Thus, $  \frac{d}{dt} G(tv)\vert_{t=0}=0$.

 We use $K$ to denote the symmetric trilinear map $d^2G(0)$, where $K(v,v,\cdot)=\frac{1}{2}\frac{d^2}{dt^2} G(tv)\vert_{t=0}$. Thus, since $G$ is smooth and the iterates are restricted to $\X$ by Assumption \ref{assump:manifold},
 a Taylor expansion gives, for $v\in R_{x_\star}^{-1}(\X)$,
 $
 G(v)=G(0)+K(v,v,\cdot)+O(\Vert v \Vert^3)
 $.
 For $x=x_\star$ and $v=\Delta_n$, this yields
 \[
 [\tp{x_\star}{x_n}]^{-1}\te{x_\star}{x_n}= I+K(\Delta_n,\Delta_n,\cdot)+O(\Vert \D_n \Vert^3).
 \]
 Lastly, as
  $  [\te{x_\star}{x_n}]^{-1}\tp{x_\star}{x_n} = \left([\tp{x_\star}{x_n}]^{-1} \te{x_\star}{x_n} \right)^{-1} = \left( I + K(\Delta_n,\D_n, \cdot) + O(\norm{\Delta_n}^3)) \right)^{-1} = I -  K(\Delta_n, \D_n,\cdot)  + O(\norm{\Delta_n}^3)$
  the conclusion follows.
In the special case the exponential map is used as retraction, \citet[][Theorem A.2.9]{waldmann2012geometric}  directly relates $K$ to the Riemann curvature tensor. They show $K(v,v,\cdot)=-\frac{1}{6}R_{x_\star}(v,\cdot)v$ for $v\in T_{x_\star}\M$. However the result by \citet{waldmann2012geometric} provides the Taylor expansion up to arbitary order in $\Vert v \Vert$.
\end{proof}
Assembling Lemmas  \ref{lem:tangent_rec}, \ref{lem:tangent_rec_2} and \ref{lem:tangent_rec_3} we obtain the desired linear recursion:
\begin{theorem}\label{thm:linear}
  Let Assumptions \ref{assump:manifold}, \ref{assump:HessianLip}, \ref{assump:noiseunbiased},  and \ref{assump:noiseLip} hold. If $\Delta_n = R_{x_{\star}}^{-1}(x_n)$ for a sequence of iterates evolving as in \eq{grad_desc}, then there exists a martingale-difference sequence $\{ \xi_{n} \}_{n \geq 0}$ satisfying  $\E[\xi_{n+1}\vert \F_{n}]=0$,
$ \E[\Vert \xi_{n+1}\Vert^2 \vert \F_{n}] = O(\norm{\D_n}^2) $, and an error sequence $\{ e_n \}_{n \geq 0}$ satisfying $\mathbb{E}[ \norm{ e_{n+1} } | \F_{n} ] \Vert = O(\norm{\Delta_n}^2 + \gamma_{n+1})$ and
$\mathbb{E}[ \norm{ e_{n+1} }^2 | \F_n ] \Vert = O(\norm{\Delta_n}^4 + \gamma_{n+1}^2)$ such that
  \[
    \D_{n+1} = \D_n - \gamma_{n+1} \Hess f(x_\star) \Delta_n  -\gamma_{n+1} \nabla f_{n+1}(x_\star) \\ -\gamma_{n+1}\xi_{n+1}  -\gamma_{n+1} e_{n+1}.
  \]
\end{theorem}
\begin{proof} Combining Lemmas \ref{lem:tangent_rec},  \ref{lem:tangent_rec_2} and \ref{lem:tangent_rec_3},
\BEAS
  \D_{n+1} &=& \D_n - \gamma_{n+1} [\te{x_\star}{x_n}]^{-1} (\nabla f_{n+1}(x_n)) + \gamma_{n+1} g_n \\
   &=& \D_n - \gamma_{n+1} [\tp{x_n}{x_\star}\te{x_\star}{x_n}]^{-1} \tp{x_n}{x_\star}(\nabla f_{n+1}(x_n)) + \gamma_{n+1} g_n \\
   &=& \D_n - \gamma_{n+1} [I -  K(\Delta_n,\D_n, \cdot) ] \circ (\Hess f(x_\star)\Delta_n + \nabla f_{n+1}(x_\star)+\tilde\xi_{n+1}+\tilde e_{n+1}) \\
   && + \gamma_{n+1} g_n + O( \gamma_{n+1} \Vert \D_{n}\Vert^3)\\
   &=& \D_n - \gamma_{n+1} \Hess f(x_\star)\Delta_n  -\gamma_{n+1} \nabla f_{n+1}(x_\star)\\
   &&-\gamma_{n+1}\tilde \xi_{n+1}  +{\gamma_{n+1}} K(\Delta_n,\Delta_n,\nabla f_{n+1}(x_\star)+\tilde \xi_{n+1} ) \\
      &&-\gamma_{n+1}\tilde e_{n+1} +{\gamma_{n+1}} K(\Delta_n,\Delta_n,\Hess f(x_\star) \Delta_n +\tilde e_{n+1}) \\
      && + \gamma_{n+1} g_n + O( \gamma_{n+1} \Vert \D_{n}\Vert^3).
 \EEAS
 Let $\xi_{n+1} = \tilde \xi_{n+1}  -{\gamma_{n+1}} K(\Delta_n,\Delta_n,\nabla f_{n+1}(x_\star)+\tilde \xi_{n+1} )$. By linearity of the map $K(\D_n,\D_n,\cdot)$,  $\E[\xi_{n+1}\vert \F_{n}]=0$. Moreover by smoothness of the retraction, the tensor $K$ is uniformly bounded in injective norm on the compact set $R_{x_\star}^{-1}(\X)$, so $\E[\Vert \xi_{n+1}\Vert^2 \vert \F_{n}] = O(\Vert \D_n\Vert^2)$.

 Let $e_{n+1}=\tilde{e}_{n+1} - K(\Delta_n,\Delta_n,\nabla^2 f(x_\star)+\tilde{e}_{n+1} ) - g_{n} +O( \Vert \D_{n}\Vert^3)$. Using Assumptions \ref{assump:manifold}, \ref{assump:noiseLip} and the almost sure upper bound on $\tilde {e}_{n+1}$ we have that this term satisfies
 \[
 \E[\Vert e_{n+1}\Vert ^2 | \F_n] = \O \left(\norm{\Delta_n}^4 + \gamma_{n+1}^2 \right).
 \]
\end{proof}
Note that sharper bounds may be obtained under higher-order assumptions on the moments of the noise.  This would provide a sharp constant of the leading asymptotic term of $O(\frac{1}{n})$, when the step-size $\gamma_n=\frac{1}{\sqrt{n}}$ is used.

\subsection{Proofs in \mysec{pfsketch2}} \label{sec:conv_rates}
Here we provide proofs, in the Euclidean setting, of both asymptotic and non-asymptotic Polyak-Ruppert-type averaging results. We apply these results to the tangent vectors $\Delta \in T_{x_\star}\M$ as described in \mysec{pfsketch1}.
\subsubsection{Asymptotic Convergence}
Throughout this section, we will consider a general linear recursion perturbed by a remainder term $e_n$ of the form:
\begin{align}
  \D_{n}=\D_{n-1} -\gamma_n A \D_{n-1}+ \gamma_n (\eps_n+\xi_{n}+e_{n}), \label{eq:rec_with_error1}
\end{align}
for which we will show an asymptotic convergence result under appropriate conditions.  

Note that we eventually apply these convergence results to iterates $\Delta_n \in T_{x_\star} \M$, which is a finite-dimensional vector space. In this setting, a probability measure can be defined on a vector space (with inner product) with a covariance operator implicitly depending on the inner product (via the dual map). 

We make the following assumptions on the structure of the recursion:
\begin{assumption} \label{assump:psd}
  $A$ is symmetric positive-definite matrix.
\end{assumption}
\begin{assumption} \label{assump:noise1}
    The noise process $\{ \eps_{n} \}$ is a martingale-difference process (with  $\E[\eps_n\vert\mathcal F_{n-1}]=0$  and  $\sup_n \E[\eps_n^2] < \infty$), for which    there exists $C>0$ such that $\E [\Vert \eps_n \Vert^4 | \mathcal{F}_{n-1}] \leq C$ for all $n \geq 0$ and a matrix $\Sigma \succ 0$ such that
    \[\E[\eps_n\eps_n^\top\vert\mathcal F_{n-1}]\overset{P}{\to} \Sigma.\]
\end{assumption}
\begin{assumption} \label{assump:noise2}
    The noise process $\{ \xi_{n} \}$ is a martingale-difference process (with  $\E[\xi_n\vert\mathcal F_{n-1}]=0$ and $\sup_n \E[\xi_n^2] < \infty$), and for sufficiently large $n \geq N$, there exists $K > 0$ such that
    \[ \E[\Vert \xi_n\Vert^2\vert\mathcal F_{n-1}] \leq K\gamma_n \text{ a.s. } \]
    with $\gamma_n \to 0$ as $n \to \infty$.
\end{assumption}
\begin{assumption} \label{assump:remainder}
  For $n \geq 0$
    \[
    \E[\Vert e_n \Vert] = O(\gamma_n).
    \]
\end{assumption}
\begin{assumption}\label{assump:step_size}
  $\gamma_n \to 0$, $\frac{\gamma_n-\gamma_{n-1}}{\gamma_n} = o(\gamma_n)$ and $\sum_{j=1}^{\infty} \frac{\gamma_j}{\sqrt{j}} < \infty$.
\end{assumption}
The first two conditions in Assumption \ref{assump:step_size} require that $\gamma_n$ decrease sufficiently slowly. For example $\gamma_n = \gamma t^{-\alpha}$ with $ \frac{1}{2} < \alpha < 1$ satisfy these two conditions
but the sequence $\gamma = \gamma t^{-1}$ does not.

We can now derive the asymptotic convergence rate,
\begin{theorem} \label{thm:asymp_ave}
  Let Assumptions \ref{assump:psd}, \ref{assump:noise1}, \ref{assump:noise2}, \ref{assump:remainder} and \ref{assump:step_size}
  hold for the perturbed linear recursion in Equation \eqref{eq:rec_with_error1}.
  Then,
\[
\sqrt n \bar{\Delta}_n  \overset{D}{\to} \mathcal N (0,  A^{-1}\Sigma A^{-1}).
\]
\end{theorem}
\begin{proof}
The argument mirrors the proof of Theorem 2 in \citet{polyak1992acceleration} so we only
sketch the primary points. Throughout we will use $C$ to denote an unimportant, global constant that may change line to line.

Consider the purely linear recursion of the form:
\begin{align}
  & \D^1_{n}=\D^1_{n-1} -\gamma_n A \D^1_{n-1}+ \gamma_n (\eps_n+\xi_{n}) \label{eq:linear_rec} \\
  & \bar{\D}^1_{n} = \frac{1}{n} \sum_{i=0}^{n-1} \D^1_{i}, \nonumber
\end{align}
which satisfies $\D^1_{0} = \D_{0}$, and approximates the perturbed recursion in Equation \eqref{eq:rec_with_error1},
\begin{align}
  & \D_{n}=\D_{n-1} -\gamma_n A \D_{n-1}+ \gamma_n (\eps_n+\xi_{n}) + \gamma_n e_{n} \label{eq:linear_rec_perturb} \\
  & \bar{\D}_{n} = \frac{1}{n} \sum_{i=0}^{n-1} \D_{i}. \nonumber
\end{align}
Now, note that we can show that
$\lim_{K \to \infty} \lim \sup_{n} \E \left[\Vert \eps_n \Vert^2 \mathbb{I}[\Vert \eps_n \Vert > K]|\F_{n-1} \right] \overset{p}{\to} 0$, using our (conditional) 4th-moment bound and the (conditional) Cauchy-Schwarz/Markov inequalities, so the relevant assumption in \citet{polyak1992acceleration} is satisfied. Then as the argument in Part 3 of the proof of Theorem 2 in
\citet{polyak1992acceleration} shows, under Assumptions \ref{assump:psd}, \ref{assump:noise1}, \ref{assump:noise2}
the conditions of Proposition (a) of Theorem 1 in \citet{polyak1992acceleration} also hold. This implies the linear process
satisfies:
\[
\sqrt n \bar{\D}^1_n  \overset{D}{\to} \mathcal N (0,  A^{-1}\Sigma A^{-1}). \notag
\]
We now argue that the process $\bar{\D}^1_n$ and $\bar{\D}_n$ are asymptotically equivalent in distribution. First,
since the noise process is coupled between Equations \ref{eq:linear_rec} and \ref{eq:linear_rec_perturb}, the differenced process
obeys a simple (perturbed) linear recursion,
\[
  \D_n - \D^1_{n} = (I-\gamma_j A)(\D_{n-1}-\D^1_{n-1}) - \gamma_n e_n. \notag
\]
Expanding and averaging this recursion (defining $\delta_n = \bar{\D}_n-\bar{\D}^1_{n}$) gives:
\begin{align}
   & \D_n-\D^1_{n} = \sum_{j=1}^{n} \Pi_{i=j+1}^{n}(I-\gamma_j A) \gamma_j e_j \implies \delta_n = \frac{1}{n} \sum_{k=1}^{n-1} \sum_{j=1}^{k}[\Pi_{i=j+1}^{k}(I-\gamma_i A)] \gamma_j e_j \notag \\
   \implies
   & \delta_n = \frac{1}{n} \sum_{j=1}^{n-1} \left[ \sum_{k=j}^{n-1} \Pi_{i=j+1}^{k}(I-\gamma_i A) \right] \gamma_j e_j. \notag
\end{align}
We can rewrite the recursion for this averaged differenced process as:
\[
  \sqrt{n}\delta_n =  \frac{1}{\sqrt{n}} \sum_{j=1}^{n-1}(A^{-1} + w_j^n)e_j,
\]
defining,
\[
  w_j^n = \gamma_j \sum_{i=j}^{n-1} \Pi_{k=j+1}^{i} (I-\gamma_k A) - A^{-1}.
\]
Now if the step-size sequence satisfies the first two conditions of Assumption \ref{assump:step_size}, by Lemma 1 and 2 in \citet{polyak1992acceleration}
we have that $\Vert w_j^n \Vert \leq C$ uniformly. So using Assumption \ref{assump:psd} we obtain that:
\[
  \sum_{j=1}^{\infty} \frac{1}{\sqrt{j}} \Vert (A^{-1}+w_j^t) e_j\Vert \leq C \sum_{j=1}^{\infty} \frac{1}{\sqrt{j}} \Vert e_j \Vert.
\]
An application of the Tonelli-Fubini theorem and Assumption \ref{assump:remainder} then shows that
\[
  \E[\sum_{j=1}^{\infty} \frac{1}{\sqrt{j}} \Vert e_j \Vert] = \sum_{j=1}^{\infty} \frac{1}{\sqrt{j}} \E [\Vert e_j \Vert] \leq C \sum_{j=1}^{\infty} \frac{\gamma_j}{\sqrt{j}} < \infty,
\]
by choice of the step-size sequence in Assumption \ref{assump:step_size}. Since $\sum_{j=1}^{\infty} \frac{1}{\sqrt{j}} \Vert e_j \Vert \geq 0$ and has finite expectation it must be that,
\[
  \sum_{j=1}^{\infty} \frac{1}{\sqrt{j}} \Vert e_j \Vert < \infty \implies \sum_{j=1}^{\infty} \frac{1}{\sqrt{j}}\Vert (A^{-1} + w_j^n)e_j \Vert < \infty.
\]
An application of the Kronecker lemma then shows that
\[
  \frac{1}{\sqrt{n}} \sum_{j=1}^{n-1} \Vert (A^{-1} + w_j^n)e_j \Vert \to 0 \implies \sqrt{n} \delta_n \to 0 \text{ a.s.}
\]
The conclusion of theorem follows by Slutsky's theorem.
\end{proof}
\subsubsection{Nonasymptotic Convergence}
Throughout this section, we will consider a general linear recursion perturbed by remainder terms $e_n$ of the form:
\begin{align}
  \D_{n}=\D_{n-1} -\gamma_n A \D_{n-1}+ \gamma_n (\eps_n+\xi_{n}+e_{n}). \label{eq:rec_with_error2}
\end{align}
We make the following assumptions on the structure of the recursion:
\begin{assumption} \label{assump:psd2}
  $A$ is symmetric positive-definite matrix, such that $A\succcurlyeq \mu \idm$ for $\mu > 0$.
\end{assumption}
\begin{assumption} \label{assump:noise1na}
    The noise process $\{ \eps_{n} \}$ is a martingale-difference process (with  $\E[\eps_n\vert\mathcal F_{n-1}]=0$  and  $\sup_n \E[\eps_n^2] < \infty$) and a matrix $\Sigma \succ 0$ such that
    \[\E[\eps_n\eps_n^\top\vert\mathcal F_{n-1}] \preccurlyeq \Sigma.\]
\end{assumption}
\begin{assumption} \label{assump:noise2na}
    The noise process $\{ \xi_{n} \}$ is a martingale-difference process (with  $\E[\xi_n\vert\mathcal F_{n-1}]=0$ and $\sup_n \E[\xi_n^2] < \infty$), and there exists $K > 0$ such that for $n\geq 0$
    \[ \E[\Vert \xi_n\Vert^2\vert\mathcal F_{n-1}] \leq K\gamma_n \text{ a.s. } \]
\end{assumption}
\begin{assumption} \label{assump:remainderna}
There exists $M$ such that for $n \geq 0$, $\E[\Vert e_n \Vert^2] \leq M \gamma_n^2$.
\end{assumption}
\begin{assumption} \label{assump:step_sizena}
  The step-sizes take the form $\gamma_n=\frac{C}{n^\alpha}$
  for $C>0$ and $\alpha\in[1/2,1)$.
\end{assumption}
\begin{assumption} \label{assump:slowrate2}
  There exists $C' > 0$ such that for $n \geq 0$, we have that
  \[\sqrt{\E[\norm{\Delta_n}^2]} = O(\sqrt{\gamma_n}) = C' n^{-\alpha/2} \]
\end{assumption}
Using these Assumptions we can derive the non-asymptotic convergence rate:
\begin{theorem} \label{thm:nonasymp_ave}
 Let Assumptions \ref{assump:psd2}, \ref{assump:noise1na}, \ref{assump:noise2na}, \ref{assump:remainderna}, \ref{assump:step_sizena} and \ref{assump:slowrate2} hold for the recursion in Equation \ref{eq:rec_with_error2},
  \[
      \mathbb{E}[\Vert \bar{\Delta}_n \Vert ^2] \leq \frac{1}{n} \tr [A^{-1} \Sigma A^{-1}] +  O(n^{-2\alpha}) + O(n^{\alpha-2}).
      \]
\end{theorem}
\begin{proof}
The argument mirrors the proof of Theorem 3 in \citet{moulines2011non} so we only sketch the key points.
First, since $A$ is invertible due to Assumption \ref{assump:psd2}, from Equation \ref{eq:rec_with_error2}:
\[
\Delta_{n-1}= \frac{A^{-1}(\Delta_{n-1}-\Delta_{n})}{\gamma_{n}} + A^{-1}\eps_{n}+A^{-1}\xi_{n}+A^{-1}e_{n},
\]
We now analyze the average of each of the 4 terms separately. Throughout we will use $C$ to denote an unimportant, numerical constant that may change line to line.
\begin{itemize}
\item Summing the first term by parts we obtain,
\[
\frac{1}{n}\sum_{k=1}^n\frac{A^{-1} (\Delta_{k-1}-\Delta_{k})}{\gamma_{k}}=\frac{1}{n}\sum_{k=1}^{n-1}A^{-1}\Delta_k \left(\frac{1}{\gamma_{k+1}}-\frac{1}{\gamma_{k}}\right)-\frac{1}{n\gamma_n}A^{-1}\Delta_n+\frac{1}{n\gamma_1}A^{-1}\Delta_0,
\]
and using Minkowski's inequality (in $L_2$) gives,
\[
\sqrt{\E \norm{ \frac{1}{n}\sum_{k=1}^n\frac{A^{-1}(\Delta_{k-1}-\Delta_{k})}{\gamma_{k}} }^2}  \leq
\frac{1}{n \mu} \sum_{k=1}^{n-1}\sqrt{\E \Vert \Delta_k\Vert^2} \abs{\frac{1}{\gamma_{k+1}}-\frac{1}{\gamma_{k}}} +\frac{\sqrt{\E \Vert \Delta_n\Vert^2}}{n\gamma_n  \mu } +\frac{\Vert \Delta_0 \Vert}{n\gamma_1 \mu}.
\]
Since we choose a sequence of decreasing step-sizes of the form $\gamma=\frac{C}{n^{\alpha}}$ for $\alpha \in [\frac{1}{2}, 1)$, an application of the
Bernoulli inequality shows that $\vert \gamma_{k+1}^{-1}-\gamma_{k}^{-1}\vert = C^{-1} [(k+1)^\alpha-k^\alpha]\leq  C^{-1} \alpha k^{\alpha-1}$. By assumption, we have that
$\sqrt{\E \Vert \Delta_n\Vert^2}\leq Cn^{-\alpha/2}$ so,
\BEAS
\sqrt{\E \norm{\frac{1}{n}\sum_{k=1}^n\frac{A^{-1}(\Delta_{k}-\Delta_{k})}{\gamma_{k}} }^2}  &\leq& \frac{C\alpha }{n \mu } \sum_{k=1}^{n-1}  k^{\alpha/2-1}+\frac{C}{\mu} n^{\alpha/2-1}+\frac{C}{n  \mu}\Vert \Delta_0 \Vert \\
 &\leq& \frac{2C n^{\alpha/2-1}}{ \mu }+\frac{Cn^{\alpha/2-1}}{\mu} +\frac{C\Vert \Delta_0 \Vert}{n  \mu}\\
 & \leq &\frac{3Cn^{\alpha/2-1}}{\mu} +\frac{C\Vert \Delta_0 \Vert}{n  \mu}.
\EEAS
This implies that,
\[
\E \norm{ \frac{1}{n}\sum_{k=1}^n\frac{A^{-1}(\Delta_{k-1}-\Delta_{k})}{\gamma_{k}} }^2 = O(n^{\alpha-2}). \]
\item Using the  Assumption  \ref{assump:noise1na} and the orthogonality of martingale increments we immediately obtain the leading order term as,
\[
\E \Vert A^{-1} \bar \eps_n\Vert ^2\leq\frac{1}{n}\tr [A^{-1} \Sigma A^{-1}].
\]
\item Using Assumption  \ref{assump:noise2na} and the orthogonality of martingale increments we obtain,
\[
\E \Vert  A^{-1} \bar  \xi_n \Vert^2 =\frac{1}{n^2 \mu^2} \sum_{k=1}^n \E \Vert \xi_k \Vert^2 \leq  \frac{C}{n^2 \mu^2} \sum_{k=0}^{n-1} k^{-\alpha} = O(n^{-(\alpha+1)}).
\]
\item Using the Minkowski inequality (in $L_2$), and Assumption  \ref{assump:remainderna},
we have that
\[
\E \Vert A^{-1} \bar e_{n-1}\Vert^2  \leq \left(\frac{1}{n\mu}\sum_{k=1}^n  \sqrt{\E \Vert e_k \Vert^2 }\right)^2\leq \frac{M^2}{(n \mu)^2} \left(\sum_{k=1}^n k^{-\alpha}\right)^2 \leq \frac{M^2}{\mu^2}n^{-2\alpha}.
\]
\end{itemize}
The conclusion follows.
\end{proof}
\subsubsection{On the Riemannian Center of Mass} \label{sec:com}
Note that $\bar{\Delta}_n$ is \textit{not} computable, but has an interesting interpretation as an upper bound on the Riemannian center of mass (or Karcher mean),
\[
K_n = \arg \min_{x \in \M} \frac{1}{n} \sum_{i=1}^{n} \norm{R_{x}^{-1}(x_i)}^2
\]
 of a set of iterates $\{ x_i \}_{n \geq 0}$ in $\M$. When it exists, computing $K_n$ is itself a nontrivial geometric optimization problem since it does not admit a closed-form solution in general.
See  \citet[][]{moakher2002means,bini2013computing,hosseini2015matrix} for more background on the Karcher mean problem.
If we consider a ``symmetric'' retraction $R$ satisfying for $x,y\in\X$ that $\Vert R_{x}^{-1}(y)\Vert^2=\Vert R_{y}^{-1}(x)\Vert^2$ (which is the case for the exponential map for example), then
\begin{lemma} \label{lem:karcher_mean}
  Let $\{ x_i \}_{i=0}^{n}$ be a sequence of iterates contained in $\M$ and let the retraction $R$ be symmetric, then
  \[\Vert R_{x_\star}^{-1}(K_n)\Vert^2\leq2\Vert\bar \Delta_n\Vert^2.\]
\end{lemma}
\begin{proof}
  The first-order optimality condition requires that $\nabla D(K_n) = 0$ where the manifold gradient is
  given by $\nabla D(x) = \frac{1}{n} \sum_{i=1}^{n} R_{x}^{-1}(x_i)$. Thus, $\nabla D(x_\star)=\bar \Delta_n$. By Assumption \ref{assump:manifold}, the function $D$ is $1$-retraction strongly convex. Defining the function  $g: t\mapsto D \left(R_{x_\star} (t \frac{R_{x_\star}^{-1}(K_n)}{\Vert R_{x_\star}^{-1}(K_n)\Vert })\right)$, we have that at $t_0= \Vert R_{x_\star}^{-1}(K_n)\Vert$,
  \[
     2\Vert \nabla D(x_\star)\Vert^2=2(g'(t_0)-g'(0))^2 \geq  {t_0^2} =\Vert R_{x_\star}^{-1}(K_n)\Vert^2.
     \]
\end{proof}
Therefore the Riemannian center of mass will enjoy the same convergence rate as $\bar \Delta_n$ itself.

\subsection{Proofs in \mysec{pfsketch3}}
Finally, we would like to asymptotically understand the evolution of the averaged vector $\tilde{\Delta}_n = R^{-1}_{x_\star}(\tilde{x}_n)$, where
$\tilde{x}_n$ is the online, streaming iterate average.
From \eq{ave_grad_desc} we see that $\tilde{\Delta}_{n+1} = F_{\bar{x}_n, x_\star}[\frac{1}{n+1} F^{-1}_{\bar{x}_n, x_\star}(\Delta_{n+1})] = \tilde{F}(\Delta_{n+1})$, defining
$\tilde{F}(\cdot) = F_{\bar x_n,x_\star}[\frac{1}{n+1} F_{\bar x_n,x_\star}^{-1}(\cdot)]$.

We first start with a lemma controlling $\Vert \tilde{\Delta}_n \Vert$, when $x_n$ locally converges to $x_\star$.
\begin{lemma} \label{lem:avg_iters}
  Let Assumptions  \ref{assump:slowrate} and \ref{assump:manifold}  hold. Consider $x_n$ and $\tilde{x}_n$, which are a sequence of iterates evolving as in  \eq{grad_desc} and \eq{ave_grad_desc}, and define
  $\tilde {\Delta}_{n} = R_{x_\star}^{-1}(\tilde {x}_{n})$. Then,
  $\E [\Vert \tilde{\Delta}_n \Vert^2] = O(\gamma_n)$ as well.
\end{lemma}
\begin{proof}
  By Assumption \ref{assump:manifold}, the function $x \to \Vert R_{x_\star}^{-1}(x) \Vert^2$ is retraction convex in $x$.
  Then,
  \begin{align}
    \Vert R_{x_\star}^{-1}(\tilde{x}_n) \Vert^2 = \Vert R_{x_\star}^{-1} \left(R_{\tilde{x}_{n-1}}(\frac{1}{n} R^{-1}_{\tilde{x}_{n-1}}(x_n)) \right) \Vert^2 \leq \frac{n-1}{n} \Vert R_{x_\star}^{-1} \left(x_{n-1}\right) \Vert^2 + \frac{1}{n} \Vert R_{x_\star}^{-1} \left(\tilde{x}_{n-1}\right) \Vert^2. \notag
  \end{align}
  A simple inductive argument then shows that $\Vert R_{x_\star}^{-1}(\tilde{x}_n) \Vert^2 \leq \frac{1}{n} \sum_{i=0}^{n} \Vert R^{-1}_{x_\star}(x_i)\Vert^2$. Using that $\E \Vert \Delta_n \Vert^2 = O(\gamma_n)$ (from Assumption \ref{assump:slowrate}),
  and taking expectations shows $\E[\Vert \tilde{\Delta}_n \Vert^2] \leq \frac{C}{n} \sum_{i=0}^{n} \gamma_i \leq C \gamma_n$ when we choose a step-size sequence of the form $\gamma_n = \frac{C}{n^{\alpha}}$.
  \end{proof}
Finally using an asymptotic expansion we can show that $\tilde{\Delta}_n$ and $\bar{\Delta}_n$ approach each other:
\begin{lemma} \label{lem:stream_avg_iters}
  Let Assumptions \ref{assump:slowrate} and \ref{assump:manifold} hold. As before, consider $x_n$ and $\tilde{x}_n$, which are a sequence of iterates evolving as in  \eq{grad_desc} and \eq{ave_grad_desc}, and define
  $\tilde {\Delta}_{n} = R_{x_\star}^{-1}(\tilde {x}_{n})$. Then,
  \[ \tilde{\Delta}_n=\bar{\Delta} + e_n, \]
where  $\E [\Vert e_n \Vert] = O(\gamma_n)$.
\end{lemma}
\begin{proof}
  A similar chain rule computation to Lemma \eqref{lem:tangent_rec} shows that $d \tilde{F}(\Delta) = \frac{1}{n+1}\idm_{T_{x_\star}\M}$.
  Now, in addition to $\tilde {\Delta}_{n+1} = F_{\bar{x}_n, x_\star}[\frac{1}{n+1} F^{-1}_{\bar{x}_n, x_\star}(\Delta_{n+1})] = \tilde{F}(\Delta_{n+1})$,
  we also have that $\tilde{\Delta}_n = \tilde{F}(\tilde{\Delta}_n)$ identically.
  As $\tilde{F}(\cdot)$ is a mapping between vector spaces applying a Taylor expansion to the first expression about
  $\tilde{\Delta}_n$
  gives:
  \begin{align}
    \tilde{\Delta}_{n+1} & = \tilde {\Delta}_{n} + \frac{1}{n+1}(\Delta_{n+1}-\tilde{\Delta}_n) + O(D^2 \tilde{F}(\Delta) \Vert \Delta_{n+1}-\tilde{\Delta}_n \Vert^2).
  \end{align}
  for $\Delta \in R_{x_\star}^{-1}(\X)$.
  Since $\tilde{F}$ is twice-continuously differentiable and $R_{x_\star}^{-1}(\X)$ is compact,
  direct computation of the Hessian using the chain and Leibniz rules shows
  \begin{align}
  \tilde{e}_n = O \left((n+1) D^2 \tilde{F}(\Delta) \Vert \Delta_{n+1}-\tilde{\Delta}_n \Vert^2 \right) =
  O\left((n+1) \left(\frac{1}{(n+1)^2} + \frac{1}{n+1} \right) \cdot \Vert \Delta_{n+1} - \tilde{\Delta}_n \Vert^2 \right),
  \notag
  \end{align}
  which implies that
  \[
   \E \Vert \tilde{e}_n \Vert= O({\gamma_n}),\]
  since both $\E [\Vert \Delta_n \Vert^2] = O(\gamma_n)$ and $\E [\Vert \tilde{\Delta}_n\Vert^2] = O(\gamma_n)$ by Lemma \ref{lem:avg_iters}.
  Therefore $ (n+1)  \tilde{\Delta}_{n+1}  = n   \tilde{\Delta}_{n}+ \Delta_{n+1}+e_n = \sum_{k=0}^{n+1}\Delta_k + \sum_{k=0}^{n+1} \tilde e_k \implies
  \tilde{\Delta}_{n+1} = \bar{\Delta}_{n+1} +  e_{n+1}$ where $e_{n+1}=\frac{\sum_{k=0}^{n+1} \tilde e_k}{n+1}$, and $\E[\Vert e_{n+1} \Vert]= \E\big[\big\Vert \frac{\sum_{k=0}^{n+1} \tilde e_k}{n+1} \big\Vert\big] \leq \frac{1}{n+1} \sum_{i=0}^{n} \E[\Vert \tilde e_k \Vert] = O(\gamma_n)$.
\end{proof}
This result states that the distance between the streaming average $\tilde{\Delta}_n = R_{x_\star}^{-1}(\tilde{x}_n)$ is close to the computationally intractable $\bar{\Delta}_n$ up to $O(\gamma_n)$ error.

We can prove a slightly stronger statement under a 4th-moment assumption on the iterates that follows identically to the above.
\begin{lemma} \label{lem:stream_avg_iters_4mom}
  Let Assumption \ref{assump:manifold} hold, and assume the 4th-moment bound $\E[\Vert \Delta_n \Vert^4] = O(\gamma_n^2)$. As before, consider $x_n$ and $\tilde{x}_n$, which are a sequence of iterates evolving as in  \eq{grad_desc} and \eq{ave_grad_desc}, and define
  $\tilde {\Delta}_{n} = R_{x_\star}^{-1}(\tilde {x}_{n})$. Then,
  \[ \E \left[ \Vert \tilde{\Delta}_n-\bar{\Delta} \Vert^2 \right] = O(\gamma_n^2). \]
\end{lemma}
\begin{proof}
  The proof is almost identical to the proofs of Lemma \ref{lem:avg_iters} and \ref{lem:stream_avg_iters} so we will be brief. Since the function $x \to x^2$ is convex
  and nondecreasing over positive support, using Assumption~\ref{assump:manifold}, the composition $x \to \Vert R_{x_\star}^{-1}(x) \Vert^4$ is also retraction-convex in $x$. An identical argument to the proof
  of Lemma \ref{lem:avg_iters} then shows that $\E[\Vert \Delta_n \Vert^4] = O(\gamma_n^2)$ implies $\E\left[\Vert \tilde{\Delta}_n \Vert^4\right] = O(\gamma_n^2)$. Using that $\E\left[\Vert \tilde{\Delta}_n \Vert^4\right] = O(\gamma_n^2)$, a nearly identical calculation
  to Lemma \ref{lem:stream_avg_iters} and an application of Minkowski's inequality (in $L_2$) shows that $\sqrt{\E \left[\Vert \bar{\Delta}_n - \tilde{\Delta}_n \Vert^2\right]} = O(\gamma_n)$. The conclusion follows.
\end{proof}

\section{Proofs in \mysec{application}} \label{sec:appapp}
Here we provide further discussion and proofs of results described in \mysec{application}.
\subsection{Proofs in \mysec{geostrong}}
Here we present the proofs of the slow convergence rate (both in 2nd and 4th moments) for SGD applied to geodesically-smooth and strongly-convex functions. As discussed
in \mysec{geostrong} we will take the retraction $R$ to be the exponential map throughout this section.
Before we begin, we recall the following Lemma from \citet{moulines2011non},
 \begin{lemma}\label{lem:computsum} Let $n,m\in\mathbb{N}$ such that $m<n$ and $\alpha\geq0$. Then,
\[
\frac{1}{2(1-\alpha)}[n^{1-\alpha}-m^{1-\alpha}]\leq \sum_{k=m+1}^n n^{-\alpha}\leq \frac{1}{1-\alpha}[n^{1-\alpha}-m^{1-\alpha}].
\]
\end{lemma}
This follows by simply bounding sums via integrals.

With this result we can now show that SGD applied to (local) geodesically-smooth and strongly-convex functions will converge at the ``slow'' rate
with an appropriately decaying size.
\begin{proposition}\label{prop:mom2}
Let Assumptions \ref{assump:manifold}, \ref{assump:noiseunbiased},  \ref{assump:noiseLip} and \ref{assump:strongconv}
  hold for the iterates evolving in  \eq{grad_desc}. Recalling that $\gamma_n=Cn^{-\alpha}$ where $C>0$ and $\alpha\in[1/2,1)$ we have,
\[
\E[d^2(x_{n},x_\star)]\leq \frac{2C  \zeta \upsilon^2}{\mu n^\alpha } + O( \exp(-c\mu n^{1-\alpha} )),
\]
and
\[
\E[d^4(x_{n},x_\star)]\leq\frac{4C(3+\zeta)\zeta\upsilon^4}{\mu n^{2\alpha} } + O( \exp(-c\mu n^{1-\alpha} )).
\]
for some $c > 0$, where $\zeta > 0$ is a constant depending on the geometry of $\M$.
\end{proposition}
\begin{proof}
Throughout we will use $c$ to denote a global, positive constant that may change from line to line. The quantity $\zeta \equiv \zeta(\kappa, c) = \frac{\sqrt{\abs{\kappa}} c}{\tanh(\sqrt{\abs{\kappa}} c)}$ is a geometric quantity from \citet{zhang2016first}, where $\kappa$ denotes the sectional curvature of the manifold.  Note that bound on the sectional curvature is subsumed by Assumption \ref{assump:manifold} -- since it is a smooth function on $\X$, so is bounded.

\paragraph{Bound on the second moment.}
We first prove the 2nd-moment bound by following the proof of Theorem 2 by \citet{moulines2011non}, but adapting it to the setting of g-strong convexity.
Using Corollary 8 (a generalization of the law of cosines to the manifold setting) by \citet{zhang2016first} we have
\begin{equation}\label{eq:mom2}
d^2(x_{n+1},x_\star) \leq d^2(x_{n},x_\star) +2\gamma_{n+1}\langle \nabla f_{n+1}(x_n),\Exp_{x_n}^{-1}(x_\star)\rangle +\gamma_{n+1}^2 \zeta \Vert  \nabla f_{n+1}(x_n) \Vert^2,
\end{equation}
where $\zeta$ satisfies $\max_{x\in\X} \zeta(\kappa, d(x,x_\star))\leq \zeta$.
Taking conditional expectations yields
\[
\E[d^2(x_{n+1},x_\star)\vert \F_n]\leq d^2(x_{n},x_\star) +2\gamma_{n+1}\langle \nabla f(x_n),\Exp_{x_n}^{-1}(x_\star)\rangle +\gamma_{n+1}^2 \zeta  \E[\Vert  \nabla f_{n+1}(x_n) \Vert^2\vert \F_n].
\]
Using the definition of g-strong convexity and Assumptions~\ref{assump:noiseLip},\ref{assump:strongconv} we directly get
 \[
\E[d^2(x_{n+1},x_\star) | \F_n]\leq (1-2\gamma_{n+1}\mu )d^2(x_{n},x_\star)+ \gamma_{n+1}^2 \zeta \upsilon^2.
\]
Taking the full expectation, and denoting by $\delta_{n}=\E[d^2(x_{n},x_\star)]$, we obtain the recursion,
\[
\delta_n\leq (1-2\gamma_{n}\mu )\delta_{n-1}+ \gamma_{n}^2 \zeta \upsilon^2.
\]
Unrolling the recursion we have,
\[
\delta_n\leq \prod_{i=1}^n(1-2\gamma_{i}\mu )\delta_0 + \zeta\upsilon^2 \sum_{i=1}^n \gamma_i^2 \prod_{k=i+1}^n(1-2\mu\gamma_k).
\]
Using the elementary inequality $(1-x)\leq \exp(-x)$ for $x\in\mathbb R$, we observe the first term on the right side decreases exponentially fast. To analyze the second term, we split it into two components around $\lfloor n/2 \rfloor$:
\begin{equation}
\sum_{i=1}^n \gamma_i^2 \prod_{k=i+1}^n(1-2\mu\gamma_k)
=
 \sum_{i=1}^{\lfloor n/2 \rfloor} \gamma_i^2 \prod_{k=i+1}^n(1-2\mu\gamma_k)
 +\sum_{i={\lfloor n/2 \rfloor}+1}^n \gamma_i^2 \prod_{k=i+1}^n(1-2\mu\gamma_k). \label{eq:split}
\end{equation}
For the first term in \eq{split}, using again $(1-x)\leq \exp(-x)$ for $x\in\mathbb R$
\BEAS
 \sum_{i=1}^{\lfloor n/2 \rfloor} \gamma_i^2 \prod_{k=i+1}^n(1-2\mu\gamma_k)
 &\leq&
 \prod_{k=\lfloor n/2 \rfloor+1}^n(1-2\mu\gamma_k)  \sum_{i=1}^{\lfloor n/2 \rfloor} \gamma_i^2\\
  &\leq&
 \prod_{k=\lfloor n/2 \rfloor+1}^n\exp(-2\mu\gamma_k)  \sum_{i=1}^{\lfloor n/2 \rfloor} \gamma_i^2\\
  &\leq&
 \exp(-2\mu \sum_{k=\lfloor n/2 \rfloor+1}^n\gamma_k)  \sum_{i=1}^{\lfloor n/2 \rfloor} \gamma_i^2\\
   &\leq&
 C \exp(-{c\mu }n^{1-\alpha}) n^{1-2\alpha},
\EEAS
using Lemma \ref{lem:computsum}, which will decrease exponentially fast as $n \to \infty$. For the second term,
\BEAS
\sum_{i={\lfloor n/2 \rfloor}+1}^n \gamma_i^2 \prod_{k=i+1}^n(1-\mu\gamma_k)&\leq&\gamma_{\lfloor n/2 \rfloor}\sum_{i={\lfloor n/2 \rfloor}+1}^n \gamma_i \prod_{k=i+1}^n(1-\mu\gamma_k)\\
&=&\gamma_{\lfloor n/2 \rfloor}\sum_{i={\lfloor n/2 \rfloor}+1}^n\frac{ 1-(1-\mu\gamma_i)}{\mu} \prod_{k=i+1}^n(1-\mu\gamma_k) \\
&=&\frac{\gamma_{\lfloor n/2 \rfloor}}{\mu}\sum_{i={\lfloor n/2 \rfloor}+1}^n [ \prod_{k=i+1}^n(1-\mu\gamma_k)-\prod_{k=i}^n(1-\mu\gamma_k)]\\
&\leq&\frac{\gamma_{\lfloor n/2 \rfloor}}{\mu}[1-\prod_{k=\lfloor n/2 \rfloor+2}^n(1-\mu\gamma_k)]]\\
&\leq&\frac{\gamma_{\lfloor n/2 \rfloor}}{\mu}\leq \frac{2C}{n^\alpha\mu}.
\EEAS
The bound on the second moment follows from this last inequality.
\paragraph{Bound on the fourth moment.}
We now prove the bound on the 4th-moment. We start by expanding the square of \eq{mom2},
\begin{multline*}
d^4(x_{n+1},x_\star) \leq d^4(x_{n},x_\star) +4\gamma_{n+1}^2(\langle \nabla f_{n+1}(x_n),\Exp_{x_n}^{-1}(x_\star)\rangle)^2 +\gamma_{n+1}^4 \zeta^2 \Vert  \nabla f_{n+1}(x_n) \Vert^4 \\ +4 \gamma_{n+1}\langle \nabla f_{n+1}(x_n),\Exp_{x_n}^{-1}(x_\star)\rangle d^2(x_{n},x_\star)  +2 \gamma_{n+1}^2 \zeta \Vert  \nabla f_{n+1}(x_n) \Vert^2d^2(x_{n},x_\star) \\+ 4 \gamma_{n+1}^3\langle \nabla f_{n+1}(x_n),\Exp_{x_n}^{-1}(x_\star)\rangle \zeta \Vert  \nabla f_{n+1}(x_n) \Vert^2.
\end{multline*}
Taking conditional expectations and using Cauchy-Schwarz we have,
\BEAS
\E [d^4(x_{n+1},x_\star) \vert\mathcal F_{n}]
&\leq&
 d^4(x_{n},x_\star) +2(2+\zeta)\gamma_{n+1}^2 \E [ \Vert \nabla f_{n+1}(x_n)\Vert^2 \vert\mathcal F_{n}]d^2(x_{n},x_\star) \\
 &&  +\gamma_{n+1}^4 \zeta^2 \E[ \Vert  \nabla f_{n+1}(x_n) \Vert^4  \vert\mathcal F_{n}] +4 \gamma_{n+1}\langle \nabla f(x_n),\Exp_{x_n}^{-1}(x_\star)\rangle d^2(x_{n},x_\star)  \\
 &&+ 4 \gamma_{n+1}^3 \zeta \E \Vert  \nabla f_{n+1}(x_n) \Vert^3 \vert\mathcal F_{n}] d(x_{n},x_\star).
\EEAS
Using that $f$ is g-strongly convex (Assumption~\ref{assump:strongconv}), the 4th-moment bound in Assumption~\ref{assump:noiseLip}, and Jensen's inequality we obtain,
\begin{multline*}
\E [d^4(x_{n+1},x_\star) \vert\mathcal F_{n}]\leq  (1-4\gamma_{n+1}\mu)d^4(x_{n},x_\star) +2(2+\zeta)\gamma_{n+1}^2 d^2(x_{n},x_\star) \upsilon^2 \\ +\gamma_{n+1}^4 \zeta^2 \upsilon^4 + 4 \gamma_{n+1}^3 \zeta \upsilon^3 d(x_{n},x_\star).
\end{multline*}
Using the upper bound $4 \gamma_{n+1}^3 \zeta \upsilon^3 d(x_{n},x_\star)  \leq 2 \gamma_{n+1}^4 \zeta^2 \upsilon^4+2  \gamma_{n+1}^2 \upsilon^2d(x_{n},x_\star) ^2 $, we have,
\begin{equation}\label{eq:dada}
\E [d^4(x_{n+1},x_\star) \vert\mathcal F_{n}]\leq  (1-4\gamma_{n+1}\mu)d^4(x_{n},x_\star) +2(3+\zeta)\upsilon^2\gamma_{n+1}^2 d^2(x_{n},x_\star)  +3\gamma_{n+1}^4 \zeta^2 \upsilon^4.
\end{equation}
Now let us define, $a_n= \E [d^4(x_{n+1},x_\star)]$,  $b_n= \E [d^2(x_{n+1},x_\star)]$ and  $u_n=a_n+\frac{2(3+\zeta)\upsilon^2}{\mu} \gamma_{n+1}b_n$. Taking the full expectation of \eq{dada}, we can bound $u_{n+1}$ as,
\BEAS
u_{n+1} &\leq &   (1-\gamma_{n+1}\mu) u_n  +3\gamma_{n+1}^4 \zeta^2 \upsilon^4 +\frac{2(3+\zeta)\zeta\upsilon^4}{\mu} \gamma_{n+1}^3 +2(3+\zeta)\upsilon^2\gamma_{n+1}^2 b_n\\
&&   -  (1-\gamma_{n+1}\mu) \frac{2(3+\zeta)\upsilon^2}{\mu} \gamma_{n+1} b_n+(1-2\gamma_{n+1}\mu) \frac{2(3+\zeta)\upsilon^2}{\mu} \gamma_{n+1} b_n.
\EEAS
Noting that $2(3+\zeta)\upsilon^2\gamma_{n+1}^2 b_n
   -  (1-\gamma_{n+1}\mu) \frac{2(3+\zeta)\upsilon^2}{\mu} \gamma_{n+1} b_n+(1-2\gamma_{n+1}\mu) \frac{2(3+\zeta)\upsilon^2}{\mu} \gamma_{n+1} b_n=0$, we obtain the simple upper-bound on $u_{n+1}$,
\[
u_{n+1} \leq    (1-\gamma_{n+1}\mu) u_n  +3\gamma_{n+1}^4 \zeta^2 \upsilon^4 +\frac{2(3+\zeta)\zeta\upsilon^4}{\mu} \gamma_{n+1}^3.
\]
Using $(1-x)\leq \exp( -x)$ for $x\in\mathbb R$, we have,
\[
u_{n+1} \leq   \exp (-\gamma_{n+1}\mu) u_n  +3\gamma_{n+1}^4 \zeta^2 \upsilon^4 +\frac{2(3+\zeta)\zeta\upsilon^4}{\mu} \gamma_{n+1}^3.
\]
We can unroll this recursion as before,
\[
u_{n} \leq   \exp (-\mu \sum_{i=1}^n \gamma_{i}) u_0  + \sum_{i=1}^n[3\zeta^2 \upsilon^4\gamma_{i}^4 +\frac{2(3+\zeta)\zeta\upsilon^4}{\mu} \gamma_{i}^3] \prod_{k=i+1}^n\exp(-\mu\gamma_k).
\]
Proceeding exactly as in the proof of the bound on the second moment, we may bound $u_n$ as
\[ u_n\leq \frac{2(3+\zeta)\zeta\upsilon^4}{\mu} \gamma_{\lfloor \frac{n}{2}\rfloor}^2 +\text{exponentially small  remainder terms}. \]
The conclusion follows.
\end{proof}

\section{Streaming PCA} \label{sec:stream_pca_app}
Given a sequence of i.i.d.~symmetric random matrices $H_n \in \rb^{d \times d}$ such that $\E H_n = H$, in the streaming $k$-PCA problem we hope to approximate the subspace of the top $k$ eigenvectors. Let us denote by $\{\lambda_i\}_{1\leq i\leq d}$ the eigenvalues of $H$ sorted in decreasing order. Sharp convergence rates and finite sample guarantees for streaming PCA (with $k=1$) were first obtained by \citet{jain2016streaming,shamir16b} using the randomized power method (with and without a positive eigengap $\lambda_1-\lambda_2$). When $\lambda_1>\lambda_2$,  \citet{jain2016streaming} showed
with proper choice of learning rate $\eta_i \sim  {\tilde{O}} \left( \frac{1}{(\lambda_1-\lambda_2)i} \right)$,  an $\epsilon$-approximation to the top eigenvector $v_1$ of $H$ could be found in $ {\tilde{O}}(\frac{\lambda_1}{(\lambda_1-\lambda_2)^2} \frac{1}{\epsilon})$ iterations with constant probability. In the absence of an eigengap, \citet{shamir16b} showed a slow rate of convergence $ {\tilde{O}}(\lambda_1/\sqrt{n})$ for the objective function using a step-size choice of $O(1/\sqrt{n})$. \citet{AllenLi2017-streampca,shamir2016fast}\footnote{\citet{shamir2016fast} does not directly address the streaming setting but his result can be extended to the streaming setting, as remarked by \citet{AllenLi2017-streampca}.} later extended these results to the more general streaming $k$-PCA setting (with $k\geq1$).

The aforementioned results are quite powerful---because they are \textit{global} convergence results. In particular, they hold for any random initialization and do not require an initialization very close to the optima. In contrast, our framework only provides local results.

However, streaming $k$-PCA still provides an instructive and practically interesting application of our iterate-averaging framework. An important theme in the following analysis is to leverage to underlying Riemannian structure of the $k$-PCA problem as a Grassmann manifold.

Throughout this section we will assume that the stream of matrices satisfies the bound $\norm{H_n} \leq 1$ a.s.

\subsection{Grassmann Manifolds}

\paragraph{Preliminaries:}
We begin by first reviewing the geometry of the Grassmann manifold and proving several useful auxiliary lemmas. We denote the  Grassmann manifold $\G$, which is the set of the $k$-dimensional subspaces of a $d$-dimensional Euclidean space. Recalling the Stiefel manifold is the submanifold of the orthonormal matrices $\{ X\in\mathbb{R}^{d\times k}, X^\top X=I_k\}$), $\G$ can be viewed as the Riemannian quotient manifold of the Stiefel manifold where two matrices are identified as equivalent when their columns span the same subspace. Finally, $\G$ can also be identified with the set of rank $k$ projection matrices $\G=\{ X \in \mathbb{R}^{d\times d} \text{ s.t. } X^\top=X, X^2=X, \tr(X)=k \}$ \citep[see, e.g.,][for further details]{edelman1998geometry,absil2004riemannian}. We will use $\mathbf{X}$ to denote an element of $\G$, and $X$ a corresponding member of the equivalence class associated to $\mathbf{X}$, which belongs to the Stiefel manifold. Further, the tangent space at that point $\mathbf X$ is given by $T_{\mathbf X}\G= \{ Y\in \mathbb{R}^{d\times k}, Y^\top X=0 \}$. In the following we identify $\mathbf X$ and $X$ when it is clear from the context.

For our present purposes, we consider the (second-order) retraction map
\begin{equation}\label{eq:retraproj}
R_X(V)=(X+V)[(X+V)^\top(X+V)]^{-1/2},
\end{equation}

which is projection-like mapping onto $\G$. Note we implicitly extend $R_X$ to all matrices in $\mathbb{R}^{d\times d}$, and do not consider it only defined on the tangent space $T_X\G$. For $V \in T_X\G$, we still have $R_X(V)=(X+V)[\idm_k+V^\top V]^{-1/2}$.  If $X^\top Y$ is invertible, then a short computation shows that $R_X^{-1}(Y)=(I-X^\top X X)Y(X^\top Y)^{-1}$ with $\Vert R_X^{-1}(Y)\Vert_F^2=\tr [( X^\top Y X Y^\top)^{-1}-I]$. As we argue next, $\Vert R_X^{-1}(Y)\Vert_F^2$ is in fact locally equivalent to the induced, Frobenius norm $d_F(X,Y)=2^{-1/2}\Vert X X^\top-Y Y^\top\Vert_F^2$.

\paragraph{Measuring Distance on $\G$:}

It will be useful to have several notions of distance defined on $\G$ between two representative elements $X$ and $Y$.
Let $\theta_i$ for $i=1, \hdots, k$ denote the principal angles between the two subspaces spanned by the columns of $X$ and $Y$, i.e.
$U \cos(\Theta) V^\top$ is the singular value decomposition (SVD) of $X^\top Y$, where $\Theta$ is the diagonal matrix of principal angles,
and $\theta$ the $k$-vector formed by $\theta_i$. Our first distance of interest will be the arc length (or geodesic distance):
\[
  d_A(X, Y) = \norm{\theta}_2 = \norm{\Exp^{-1}_{X}(Y)}_2,
\]
while the second will be the projected, Frobenius norm:
\[
    d_F(X, Y) = \norm{\sin \theta}_2 = 2^{-1/2} \norm{X X^\top - Y Y^\top}_{F}.
\]
The distance $d_F(\cdot, \cdot)$ is induced by embedding $\G$ in Euclidean space $\rb^{d \times k}$ and inheriting the corresponding Frobenius norm. Lastly, we will also consider the pseudo-distance induced by the retraction $\Vert R^{-1}_{X}(Y)\Vert_F$. Conveniently, for small principal angles we can show these quantities are asymptotically equivalent,
\begin{lemma}
 \label{lem:distequiv}
  Let $\theta$ denote the $k$-vector of principal angles between the subspaces spanned by $X$ and $Y$. If
  $\norm{\theta}_{\infty} \leq \frac{\pi}{4}$ then:
  \[
   \frac{\pi}{2}  \Vert R^{-1}_{X}(Y)\Vert_F \geq \frac{\pi}{2}   d_{F}(X, Y) \geq  d_{A}(X, Y) \geq \frac{1}{\sqrt{2}}  \Vert R^{-1}_{X}(Y)\Vert_F. \]
\end{lemma}
\begin{proof}
  In fact $d_{A}(X, Y) \geq d_{F}(X, Y)$ uniformly over $\theta$ which follows $\sin x \leq x$ for all $x$. Using the elementary inequality $|\frac{\sin x}{x}| \geq \frac{2}{\pi}$ for $x \in [-\pi/2, \pi/2]$ we immediately obtain that $d_{A}(X, Y) \leq \frac{\pi}{2} d_{F}(X, Y)$ if $\norm{\theta}_{\infty} \leq \pi/2$.
A direct computation shows:
\[
\Vert R^{-1}_{X}(Y)\Vert^2_F = \tr[ (X^\top YY^\top X)^{-1}-I]=  \tr[ \cos(\Theta)^{-2}-I]=\tr[ (I-\sin^2(\Theta))^{-1}-I].
\]
Using that for $x\in(0,\frac{1}{\sqrt{2}})$, $\frac{1}{1-x}\geq 1+x$ and $\frac{1}{1-x}\leq 1+2x$, we obtain for $\Vert \theta \Vert_\infty<\pi/4$ that,
\[
d^2_{F}(X, Y)\leq \Vert R^{-1}_{X}(Y)\Vert^2_F \leq 2d^2_{F}(X, Y),\]
concluding the argument.
\end{proof}
The local equivalence between these quantities will prove useful for relating various algorithms for the
streaming $k$-PCA problem.

\paragraph{Streaming $k$-PCA in $\G$:} Within the geometric framework we can cast the $k$-PCA problem as minimizing the Rayleigh quotient, $f(X) = -\frac{1}{2} \tr [X^\top H X]$, over $\G$ as
\[\min _{X \in \G} -\frac{1}{2} \tr [X^\top H X].\]
  \citet{edelman1998geometry} show the (Riemannian) gradient of $f$\footnote{Note that on embedded manifold the Riemannian gradient of a function is given by projection of its Euclidean gradient in the tangent space of the manifold at a point.} is given by
  \[\nabla f (X)=-(I-X^\top X) H X.\]
 Similarly, the Hessian operator is characterized by the property that if $\Delta\in T_X\G$, then
\[\Hess f (X)[\Delta] =\Delta X^\top H X-(\idm -X X^\top)H\Delta.\]
It is worth noting that $\nabla f (X) = 0$ for $X$ spanning any subspace or eigenvector of $H$, but that $\Hess f(X)$ is only positive-definite at the optimum $X_\star$.

\subsection{Algorithms for streaming $k$-PCA} \label{sec:alg_stream}

We now have enough background to describe several, well-known iterative  algorithms for the streaming $k$-PCA problem and elucidate their relationship to Riemannian SGD.

\begin{description}
\item[Randomized power method \citep{OjaKar85}:] corresponds to  SGD on the Rayleigh quotient (with step size $\gamma_n$) over Euclidean space followed by a projection,
\[
X_n = ( \idm_d + \gamma_n H_n ) X_{n-1}  \big[ X_{n-1}^\top ( \idm_d + \gamma_n H_n ) ^{2}X_{n-1}  \big]^{-1/2}.
\]
\item[Oja iteration \citep{Oja82}:]corresponds to a first-order expansion in $\gamma_n$ of the previous randomized power iteration,
\[
 X_n = X_{n-1}  + \gamma_n ( \idm - X_{n-1} X_{n-1}^\top) H_n  X_{n-1}.
 \]
\item[Yang iteration \citep{Yan95}:] corresponds to a symmetrization of the Oja iteration,
\[
X_{n} =  X_{n-1}  + \gamma_n (2 H_n - X_{n-1} X_{n-1}^\top H_n-  H_n X_{n-1} X_{n-1}^\top)  X_{n-1}.
\]
In the special case that $H_n=h_n h_n^\top$, the Yang iteration can also be related to the unconstrained stochastic optimization of the function $X\mapsto\E \Vert h_n- X X^\top h_n\Vert_F^2$.
\item[(Stochastic) Gradient Descent over $\G$ \citep{bonnabel2013stochastic}:] corresponds to directly optimizing the Rayleigh quotient over $\G$ by equipping SGD with either the exponential map $\Exp$ or the aforementioned retraction $R$,
\[
X_n = R_{X_{n-1}}\big(\gamma_n(I-X_{n-1} X_{n-1}^{\top}) H_n X_{n-1}\big).
\]
\end{description}
We are now in position to show that for the present problem Riemannian SGD, the randomized power method, and Oja's iteration
are equivalent updates up to $O(\gamma_n^2)$ corrections. First note that the randomized power method (with the aforementioned choice of retraction $R$) can be written as:
\begin{align}
  R_{X}(\gamma_n H_nX), \label{eq:rand_power}
\end{align}
bearing close resemblance to the Riemannian SGD update,
\begin{align}
  R_{X}(\gamma_n(I-X X^{\top}) H_nX). \label{eq:rie_sgd}
\end{align}
The principal difference between \eq{rie_sgd} and \eq{rand_power} is that in the randomized power method, the Euclidean gradient is used instead of the Riemannian gradient. In the following Lemma we argue that both of the updates in \eq{rie_sgd} and \eq{rand_power} can be
directly approximated by the Oja iteration, $X_{n+1}=X_n+\gamma_{n+1}(I-X_n X_n^{\top}) H_{n+1}X_n+O(\gamma_{n+1}^2)$ up to
2nd-order terms and hence that $\Vert R_{X_\star}^{-1}(X_{n+1})) -R_{X_\star}^{-1}(R_{X_n}(\gamma_n\nabla f_{n+1}(X_n))))\Vert_F^2=O(\gamma_n^2)$. Therefore, a direct modification of Lemma \ref{lem:tangent_rec} shows that the
iterates in \eq{rand_power} generated from the randomized power method may also be linearized in the tangent space.

\begin{lemma} \label{lem:equiv_oja}
  Alternatively, let $X_n = R_{X_{n-1}}(\gamma_n(I-X_{n-1} X_{n-1}^{\top}) H_n X_{n-1})$ denote the Riemannian SGD update (equipped with second-order retraction $R$) or $X_n = R_{X_{n-1}}(H_n X_{n-1})$, the randomized power update. Then both updates satisfy,
  \[
    X_{n}=X_n+\gamma_{n}(I-X_{n-1} X_{n-1}^{\top}) H_{n}X_{n-1}+O(\gamma_{n}^2),
  \]
  and hence are equivalent to the Oja update up to $O(\gamma_n^2)$ terms.
\end{lemma}
\begin{proof}
The computation for both the randomized power method and the Riemannian SGD (equipped with the retraction $R$) are straightforward. For the randomized power method we have that,
\BEAS
 X_n & = &
  ( \idm_d + \gamma_n H_n ) X_{n-1}  \big[ X_{n-1}^\top ( \idm_d + \gamma_n H_n ) ^{2}X_{n-1}  \big]^{-1/2} \\
& = &
 ( \idm_d + \gamma_n H_n ) X_{n-1}  \big[ \idm_d + 2 \gamma_n X_{n-1}^\top   H_n  X_{n-1} \big]^{-1/2} + O(\gamma_n^2) \\
& = &  ( \idm_d + \gamma_n H_n ) X_{n-1}  \big[ \idm_d -  \gamma_n X_{n-1}^\top   H_n  X_{n-1}\big] + O(\gamma_n^2)\\
& = & X_{n-1}+\gamma_n [\idm_d -X_{n-1}X_{n-1}^\top]  H_n X_{n-1}+ O(\gamma_n^2).
\EEAS
An identical computation shows the same result for the Riemannian SGD update,
\begin{align}
  X_n &  =
  \big( \idm_d +  \gamma_n  [\idm_d -X_{n-1}X_{n-1}^\top] H_n \big) X_{n-1}  \big[ X_{n-1}^\top( \idm_d +  \gamma_n  [\idm_d -X_{n-1}X_{n-1}^\top] H_n )  ^{2}X_{n-1}  \big]^{-1/2} \notag \\
& =  X_{n-1}+\gamma_n [\idm_d -X_{n-1}X_{n-1}^\top]  H_n X_{n-1}+ O(\gamma_n^2). \notag
\end{align}
\end{proof}
Since these two algorithms are identical up to $O(\gamma_n^2)$ corrections we can directly show they will have the same linearization in $T_{X_{\star}} \M$.
\begin{lemma} \label{lem:oja_linear}
  Let, $\Delta_n=R_{X_\star}^{-1}(X_{n})$, where $X_n$ is obtained from one iteration of the randomized power method, $X_n=R_{X_n}(\gamma_n H_n X_n)$. Then $\Delta_n$ obeys,
  \[
    \Delta_n=\D_{n-1} -\gamma_n \Hess f(x_\star) \D_{n-1}+ \gamma_n (\eps_n+\xi_{n}+e_{n}) +O(\gamma_n^2).
  \]
\end{lemma}
\begin{proof}
 Let $\dot \Delta_n=R_{X_\star}^{-1}(Y_{n})$, where $Y_n$ is obtained from one iteration of Riemannian SGD (with the aforementioned second-order retraction $R$), $Y_n=R_{X_n}(\gamma_n \nabla f_n(X_n))$. From Lemma \ref{lem:tangent_rec_3} we have that $\dot \Delta_n$ may be linearized as
\[
 \dot\D_{n}=\D_{n-1} -\gamma_n \Hess f(x_\star) \D_{n-1}+ \gamma_n (\eps_n+\xi_{n}+e_{n})
\]
Defining $\square= X_{n-1}+\gamma_n [\idm_d -X_{n-1}X_{n-1}^\top]  H_n X_{n-1}$ for the Oja update, we can then show the randomized power method satisfies,
\[
R^{-1}_{X_\star}(X_n)= Y_n(X_\star^\top Y_n)^{-1}=[\square +O(\gamma_n^2)](X_\star^\top (\square +O(\gamma_n^2)))^{-1}=\square (X_\star^\top \square )^{-1}+O(\gamma_n^2)
\]
This allows us to directly bound the difference between $\dot\D_{n}$ and $\D_n$ using Lemma \ref{lem:equiv_oja},
\[
\Vert \dot \Delta_n- \Delta_n \Vert_F=
\Vert R^{-1}_{X_\star}(X_n)-R^{-1}_{X_\star}(Y_n) \Vert_F =
\Vert Y_n(X_\star^\top Y_n)^{-1}-X_n(X_\star^\top X_n) \Vert_F = O(\gamma_n^2).
\]
Hence the randomized power method iterate $X_n$ obeys the linearization,
\[
\Delta_n=\D_{n-1} -\gamma_n \Hess f(x_\star) \D_{n-1}+ \gamma_n (\eps_n+\xi_{n}+e_{n}) +O(\gamma_n^2),
\]
and falls within the scope of our framework.
\end{proof}

\subsection{Algorithms for Streaming Averaging}
Just as there are several algorithms to compute the primal iterate $X_n$, their are several reasonable ways to compute the streaming average $\tilde X_n$. Our general framework directly considers $\tilde X_n=R_{\tilde X_{n-1}}[\frac{1}{n}R^{-1}_{\tilde X_{n-1}}(X_n)]$ which leads to the following update rule:
\begin{equation}\label{eq:averageojaret}
\tilde X_n=R_{\tilde X_{n-1}}\Big[\frac{1}{n}(\idm -\tilde X_{n-1}\tilde X_{n-1}^\top)X_n[\tilde X_{n-1}X_n]^{-1}\Big].
\end{equation}
As described in \mysec{com}, the streaming average is an approximation to a corresponding global Riemannian average (which is intractable to compute). Hence, it is reasonable to consider
other global Riemannian averages that the streaming average approximates. For instance, the update rule in \eq{averageojaret} is naturally motivated by the global minimization of $X\mapsto\sum_{i=1}^n\Vert R^{-1}_X(X_i)\Vert^2$. Considering instead the distance $d_F$, and attempting to minimizing $X\mapsto\sum_{i=1}^n d^2_F(X,X_i)$, suggests a different averaging scheme. A short computation shows the aforementiond problem can be rewritten as the maximization of the function $2\tr [X^\top(\sum_{i=1}^n X_iX_i^\top)X]$ and is therefore precisely equivalent to the $k$-PCA problem.
With this in mind, we can directly use the randomized power method to compute the streaming average of the iterates. This leads to the different update rule:
\begin{equation}\label{eq:averageoja}
\tilde X_n=R_{\tilde X_{n-1}}\Big[\frac{1}{n}X_nX_n^\top \tilde X_{n-1}\Big],
\end{equation}
which is exactly one step of randomized power method iteration to compute the first $k$ eigenvectors of the matrix $\frac{1}{n}\sum_{i=1}^nX_iX_i^\top$ with step size $\gamma_n=\frac{1}{n}$.

Using similar computations to those in Lemma \ref{lem:equiv_oja} and \ref{lem:oja_linear}, it can be shown that these iterates are equivalent to those in \eq{averageojaret} since they have a similar linearzation in $T_{X_\star}\G$.
\begin{lemma}\label{lem:average_oja}
 Let $\tilde X_n=R_{\tilde X_{n-1}}\Big[\frac{1}{n}X_nX_n^\top \tilde X_{n-1}\Big]$ and $\tilde \D_n=R_{X_\star}^{-1}(\tilde X_n)$. Then $\tilde X_n$ obeys
\[
\tilde \D_n=\tilde \D_{n-1}+\frac{1}{n}\big[\tilde \D_{n-1}+ \D_{n} \big]+ \frac{1}{n}O\big( \Vert \D_{n-1}\Vert^2+\Vert \tilde\D_{n}\Vert^2 + \frac{1}{n^2}\big).
\]
\end{lemma}
\begin{proof}
Following a similar approach as in the proof of Lemma \ref{lem:tangent_rec} we can show that
\[
\tilde \D_n=\tilde \D_{n-1}+\frac{1}{n+1}[\idm-X_\star X_\star^\top]X_nX_n^\top \tilde X_{n-1}+O(\frac{1}{n^2}).
\]
Then a direct expansion shows that for all $\Delta = R^{-1}_{X_\star}(X)$,
\[
X=R_{X_\star}(\D)=(X_\star+\D)[\idm+\D^\top\D]^{-1/2}=X_\star+\D+O(\Vert\D\Vert^2),
\]
so we obtain (using $\Vert\D_n^\top \tilde\D_{n-1}\Vert=O(\Vert\D_n\Vert^2+\Vert \tilde \D_{n-1}\Vert^2)$) that
\[
[\idm-X_\star X_\star^\top]X_nX_n^\top \tilde X_{n-1}=\D_n+\tilde \D_{n-1}+O(\Vert\D_n\Vert^2+\Vert \tilde \D_{n-1}\Vert^2)).
\]
\end{proof}
This manner of iterate averaging is interesting, not only due to its simplicity, but due to its close connection to the primal, randomized power method. In fact, this averaging method allows us
to interpret the aforementioned, streaming, averaged PCA algorithm as a preconditioning method.

Consider the case where we hope to compute the principal $k$-eigenspace of a poorly conditioned matrix $H$. The aforementioned streaming, averaged PCA algorithm can be interpreted as the composition of two stages. First, running $n$ steps of the randomized power method with step size $\gamma_n\propto1/\sqrt{n}$, to produce a well-conditioned matrix $\frac{1}{n}\sum_{i=1}^nX_iX_i^\top$, and then using the randomized power method with step size $1/n$ to compute the average of the points $X_i$ (which is efficient since the eigengap of $\frac{1}{n}\sum_{i=1}^nX_iX_i^\top$ is large). The intuition for this first step is formalized in the following remark,
\begin{remark}
  Assume the sequence of iterates $\{ X \}_{i=0}^{n}$ satisfies $d_F(X_i, X_\star) =  {O}(\sqrt{\gamma})$ for all $i=1, \dots, n$ then the eigengap of the averaged matrix $\tilde{X} = \frac{1}{n} \sum_{i=1}^{n} X_i X_i^\top$ satisfies $\tilde \lambda_{k}-\tilde \lambda_{k+1} \geq 1-kO(\sqrt{\gamma})$ (where we denote by $\{\tilde \lambda_i\}_{1\leq i\leq d}$ the eigenvalues of $\tilde X$ sorted in decreasing order).
\end{remark}
\begin{proof}
  $X_i X_i^\top = X_\star X_\star^\top + \eta_i$ for $\eta_i$ satisfying $\norm{\eta_i}_{F} \leq O(\sqrt{\gamma})$ by the definition of the distance $d_{F}(\cdot, \cdot)$, so it follows that
  $\tilde{X}_n = X_\star X_\star^{\top} + \eta$ for $\eta$ satisfying $\norm{\eta}_{F} \leq O(\sqrt{\gamma})$ using the triangle inequality.
  Now using the Weyl inequalities \citep{horn1990matrix} we have that:
  \[
    \lambda_{k}(\tilde{X}_n - \eta + \eta) \geq \lambda_{k}(\tilde{X}_n - \eta) + \lambda_d(\eta),
  \]
  where we denote by $\lambda_k(M)$ the $k$-th largest eigenvalue of the matrix $M$. Moreover $\lambda_{k}(\tilde{X}_n - \eta)=1$ and $|\lambda_d(\eta)| \leq \norm{\eta}_F \leq O(\sqrt{\gamma})$ since the spectral norm is upper bounded by the Frobenius norm, so
  $\lambda_{k}(\tilde{X}_n) \geq 1-O(\sqrt{\gamma})$.
  Now, recall that each $X_i \in \G$, so $\Tr[\tilde{X}_n]= \frac{1}{n} \sum_{i=1}^{n} \Tr[X_i X_i^\top] = k$ due to the normalization constraint $X_i X_i^\top = \idm_{k}$.
 Thus we must have that $\tilde \lambda_{k+1} \leq k O(\sqrt{\gamma})$ which implies $\tilde \lambda_k(\tilde{X}_n)-\tilde\lambda_{k+1}(\tilde{X}_n) = 1-kO(\sqrt{\gamma})$.
\end{proof}
For these reasons we prefer to use \eq{averageoja} rather than \eq{averageojaret} in our experiments and our presentation. It is worth noting that since these both iterations are equivalent up to $O(\gamma_n)$ corrections, they will (a) have the same theoretical guarantees in our framework and (b) perform similarly in practice.

\subsection{Convergence Results}

We are now ready to apply Theorem \ref{thm:oja_main} to the streaming $k$-PCA problem, with only the tedious task of verifying our various technical assumptions. Since we only seek to derive a local convergence result, we will once again use Assumption \ref{assump:manifold} which stipulates the iterates $X_n$ are restricted to $\X$ and that the map $X\mapsto\norm{R_{X_\star}^{-1}(X)}_F^2$ is strongly-convex. Here we will take the set $\X=\{Y : d_F(X_\star, Y)\leq \delta \}$, for $\delta>0$.  As previously noted, if the retraction $R$ was taken as the exponential map, the map $X\mapsto \norm{\Exp_{X_\star}^{-1}(X)}_F^2$ would always be retraction strongly convex locally in a ball around $y$ whose radius depends on the curvature, as explained by \citet{Afs11}. However,
we can also verify, that even when $R$ is the aforementioned projection-like retraction it is locally retraction strongly convex,
\begin{remark} \label{rmk:strongcvxpca}
 Let $k=1$ and take $R$ as the second-order retraction defined in \eq{retraproj}. Then there exists a constant $\delta>0$ such that on the set
 $\X = \{Y : \abs{X_\star^\top Y} \geq 1-\delta \}$, $X\mapsto \norm{R_{X_\star}^{-1}(X)}_F^2$ is retraction strongly convex.
\end{remark}
\begin{proof}
 For notational convenience, let $\epsilon$ be defined so we consider $\mathcal{Y} = \{ Y : \abs{X_\star^\top Y} \geq 1-\epsilon \}$. It suffices to show that for $\norm{X_\star}_2=1$, all $Y$ such that $\norm{Y}_2=1$ and $\abs{X_\star^\top Y} \geq 1-\epsilon$, and all $V$ such that $\norm{V}_2=1$, the map $g(t) : t \mapsto \norm{R_{X_\star}^{-1}(R_{Y}(t V))}_F^2$ is convex in $t$. For $k=1$, using the previous formulas for the retraction and its inverse norm, we can explicitly compute the map as $g(t) = \frac{1+t^2}{(X_\star^\top(Y+tV))^2}-1$. We then have that $g''(t) = \frac{2 (3 (X_\star V)^2-2t (X_\star^\top v)(X_\star^\top y)+ (X_\star^\top y)^2)}{(t (X_\star^\top v) +(X_\star^\top y))^4}$. It suffices to show that $3 (X_\star^\top V)^2+ (X_\star^\top Y)^2 > 2t (X_\star^\top V)(X_\star^\top Y)$ for all $t$ such that $R_{Y}(t V) \in \mathcal{Y}$.
 Since $V \in T_{Y}\M$ we must have that $V^\top Y = 0$ which implies that $(V^\top X_\star)^2 = (V^\top (Y-X_\star))^2 \leq 2-2Y^\top X_{\star} \leq 2 \epsilon$. Hence we have that $2\abs{(X_\star^\top V)(X_\star^\top Y)} \leq 4\epsilon$. On the otherhand,
 we have that $3 (X_\star^\top V)^2+ (X_\star^\top Y)^2 > (X_\star^\top Y)^2 > (1-\epsilon)^2$. Thus the statement is satisfied for all $\abs{t} \leq \frac{(1-\epsilon)^2}{4\epsilon}$. Using the local equivalence of distances for PCA (Lemma \ref{lem:distequiv}), the conclusion holds for some constant $\delta>0$.
\end{proof}
A similar, but lengthier linear algebra computation, also shows the result for $k > 1$.

Theorem \ref{thm:oja_main} requires several other assumptions (namely Assumptions \ref{assump:HessianLip} and \ref{assump:noiseLip}) which are (surprisingly) tedious to verify even in the simple case when $f$ is the quadratic Rayleigh quotient. However, we note Lemma \ref{lem:tangent_rec_2}, for the streaming $k$-PCA problem, can also be derived from first principles, circumventing the need to directly check these assumptions.
\begin{remark} If $f$ is the Rayleigh quotient, the conclusion of Lemma \ref{lem:tangent_rec_2} follows (without having to directly verify Assumptions \ref{assump:HessianLip} and \ref{assump:noiseLip}), when the stream of matrices satisfies the almost sure bound $\Vert H_n\Vert\leq1$.
\end{remark}
\begin{proof}
As in the proof of Lemma \ref{lem:average_oja}, we use
\[
X_n=R_{X_\star}(\D_n)=(X_\star+\D_n)[\idm+\D_n^\top\D_n]^{-1/2}=X_\star+\D_n+O(\Vert\D_n\Vert^2),
\]
to simplify $\nabla f_{n+1}(X_n)$. This yields
\BEAS
\nabla f_{n+1}(X_n)
&=&(\idm-X_nX_n^\top)H_{n+1}X_n +O(\Vert\D_n\Vert^2)\\
&=&(\idm-(X_\star +\D_n)(X_\star+\D_n)^\top)H_{n+1}(X_\star +\D_n)\\
&=&(\idm-X_\star X_\star^\top)H_{n+1}X_\star+(\idm-X_\star X_\star^\top)H_{n+1}\D_n-\D_n X_\star^\top H_{n+1}X_\star\\
&&+ X_\star \D_n^\top H_{n+1}X_\star+O(\Vert\D_n\Vert^2).
\EEAS
Upon projecting back into $T_{X_\ast}\M$, the term $ X_\star \D_n^\top H_nX_\star$ vanishes and we obtain,
\BEAS
(\idm-X_\star X_\star^\top)\nabla f_{n+1}(X_n)
&=& \Hess f(X_\star)\D_n+\nabla f_{n+1}(X_\ast)+\xi_{n+1},
\EEAS
with $\xi_{n+1}=(\idm-X_\star X_\star^\top)H_{n+1}\D_n-\D_n X_\star^\top H_{n+1}X_\star$ satisfying $\E[\xi_{n+1}|\mathcal{F}_{n}]=0$ and $\E[\Vert\xi_{n+1}\Vert^2|\mathcal{F}_{n}]=O(\Vert \Delta_n\Vert^2)$ since $\E[\Vert H-H_{n+1}\Vert^2|\mathcal{F}_{n}]$ is bounded (recall that we assume $\norm{H_n} \leq 1$ a.s.).
\end{proof}
Using results from the work of \citet{AllenLi2017-streampca} and \citet{shamir2016fast} we can now argue under appropriate conditions that
the randomized power method for the streaming $k$-PCA problem will converge in expectation to a neighborhood of $X_\star$.
Namely,
\begin{lemma}[\cite{AllenLi2017-streampca} and \cite{shamir2016fast}]
 \label{lem:oja_converge}
  Let $X_n$ denote the iterates of the randomized power method evolving as in \eq{rand_power}.
  If Assumption \ref{assump:manifold} holds for $\X$ defined above with $\delta<1/4$ then,
  \[
    \E[d_{F}^2 (X_\star, X_n)] = O(\gamma_n^2).
  \]
\end{lemma}
\begin{proof}
  This a direct adaptation of Lemma 10 by \citet{shamir2016fast} as explained by \citet[][Section 3]{AllenLi2017-streampca}.
\end{proof}
Finally, we are able to present the proof of Theorem \ref{thm:oja_main}. Since the assumptions of Lemma \ref{lem:oja_converge} are satisfied by assumption, and we have asymptotic equivalence of distances from Lemma \ref{lem:distequiv}, Assumption \ref{assump:slowrate} is satisfied. Further, since Lemmas \ref{lem:oja_linear} and \ref{lem:average_oja} show the linearized process of the randomized power iterates is equivalent to that of Riemannian SGD up to $O(\gamma_n^2)$ corrections, distributional convergence immediately follows from these Lemmas and Theorem \ref{thm:main}. Lastly, we can compute the asymptotic variance.

We first compute the inverse of the Hessian of $f$. Let us consider the following basis of $T_{X_\star}\G$, $\{v_i e_j^\top \}_{k<i\leq d;j\leq k}$, where we denote by $v_i$ the eigenvector of $H$ associated with the eigenvalue $\lambda_i$ and $e_j$ the $j$th standard basis vector.
Indeed  $\{v_i e_j^\top \}_{k<i\leq d;j\leq k}$ is a linear independent set since   $\{v_i e_j^\top \}_{i\leq d;j\leq k}$ is a basis of $\mathbb R^{d,k}$. Moreover for $k<i\leq d$, $X_\star^\top v_i e_j^\top=0$ so $ v_i e_j^\top\in T_{X_\star}\G$. We conclude this set is a basis from a dimension count. We now compute the projection of the Hessian on this basis
\BEAS
\Hess f(X_\star)[v_ie_j^\top]&=& v_ie_j^\top X_\star^\top H X_\star - Hv_ie_j^\top\\
&=&
v_ie_j^\top \diag(\lambda_1,\dots, \lambda_k)-\lambda_i v_ie_j^\top\\
&=&(\lambda_j-\lambda_i) v_ie_j^\top.
\EEAS
Therefore
\[
[\Hess f(X_\star)]^{-1}[v_ie_j^\top]= \frac{v_ie_j^\top}{\lambda_j-\lambda_i}.
\]
We now reparametrize $\tilde H_n= H^{-1/2} H_n H^{-1/2} $ such that $\E[\tilde H_n]= \idm$. Thus
\BEAS
\nabla f_n(X_\star)
&=&  (\idm-X_\star X_\star ^\top)H_nX_\star \\
&=&(\idm-X_\star X_\star ^\top)H^{1/2} \tilde H_nH^{1/2} X_\star\\
&=&
\sum_{i=k+1}^d \sqrt{\lambda_i} v_iv_i^\top  \tilde H_n\sum_{j=1}^k \sqrt{\lambda_j}v_j e_j^\top\\
&=&
\sum_{j=1}^k\sum_{i=k+1}^d\sqrt{\lambda_i \lambda_j} [v_i^\top \tilde H_n v_j] v_i  e_j^\top.
\EEAS
This yields
\[
[\Hess f(X_\star)]^{-1}\nabla f_n(X_\star) = \sum_{j=1}^k\sum_{i=k+1}^d\frac{\sqrt{\lambda_i \lambda_j}}{\lambda_j-\lambda_i}[v_i^\top \tilde H_n v_j]v_ie_j^\top.
\]
And the asymptotic covariance becomes,
\begin{align}
& [\Hess f(X_\star)]^{-1}\E[\nabla f_n(X_\star)\nabla f_n(X_\star)^\top] \Hess f(X_\star)]^{-1}= \notag \\
& 
\sum_{j'=1}^k\sum_{i'=k+1}^d \sum_{j=1}^k\sum_{i=k+1}^d\frac{\sqrt{\lambda_i \lambda_j} \cdot \sqrt{\lambda_{i'} \lambda_{j'}}}{(\lambda_{j}-\lambda_{i}) \cdot (\lambda_{j'}-\lambda_{i'})} \E \left[\left(v_i^\top \tilde H_n v_j\right) \left(v_{i'}^\top \tilde H_n v_{j'}\right)\right](v_i e_j^\top) \otimes (v_{i'} e_{j'}^\top). \notag
\end{align}

It is interesting to note that the tensor structure of $C_{ii'jj'}$ significantly simplifies if we have that $H_n = h_n h_n^\top$ for $h_n \sim \mathcal{N}(0, \Sigma)$ -- that $H_n$ is comprised of a rank-one stream of Gaussians. Recall we have $\tilde{H}_n = H^{-1/2} H_n H^{-1/2} = H^{-1/2} x_n x_n^\top H^{-1/2} = x'_n (x'_n)^\top$.
So for a rank-one stream,
\begin{align}
  C_{ii'jj'} = \E \left[\left(v_i^\top \tilde H_n v_j\right) \left(v_{i'}^\top \tilde H_n v_{j'}\right)\right]  = \E \left[ \langle v_i, x'_n \rangle
  \langle v_{i'}, x'_n \rangle \langle v_j, x'_n \rangle \langle v_{j'}, x'_n \rangle \right]. \notag
\end{align}
Since $x'_n \sim \mathcal{N}(0, I_d)$ and the law of a jointly multivariate normal random variable is invariant under an orthogonal rotation, the joint distribution of the vector
$[\langle v_i, x'_n \rangle,
\langle v_{i'}, x'_n \rangle, \langle v_j, x'_n \rangle, \langle v_{j'}, x'_n \rangle] \sim \mathcal{N}(0, I_4)$ for $i \neq i'$ and $j \neq j'$. So the only non-vanishing terms become,
\begin{align}
  C_{ii'jj'} = \delta_{ii'} \delta_{jj'}, \notag
\end{align} and the asymptotic covariance reduces to,
\begin{align}
    C = \sum_{j=1}^k\sum_{i=k+1}^d\frac{{\lambda_i \lambda_j}}{(\lambda_j-\lambda_i)^2} (v_i e_j^\top) \otimes (v_{i} e_{j}^\top). \notag
\end{align}

This is precisely the same asymptotic variance given by \citet{reiss2016non} and matches the lower bound of \citet{CaiMaWu13} obtained for the (Gaussian) spiked covariance model. However, formally, our result requires the $H_n$ to be a.s. bounded.

\section{Additional Experiments: Numerical Counterexample to Oja Convergence with Constant Step Size}
\label{sec:counter}
 \begin{figure}[!ht]
\centering
  \includegraphics[width=0.8\linewidth]{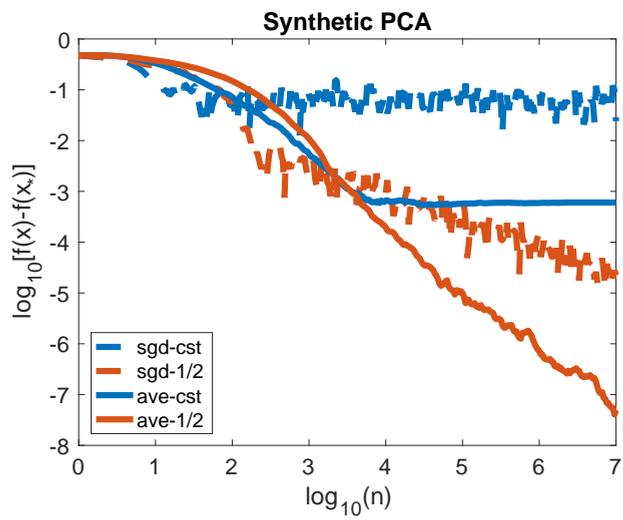}
  \vspace{-3.0cm}
  \caption{Counterexample for the convergence of averaged SGD with constant step size.}
     \label{fig:counter}
\end{figure}
In this section we present an additional experiment which shows empirically that averaged SGD with constant step-size does not always converge for the streaming $k$-PCA problem. We consider $d=2$ and  a covariance matrice $H$ with random eigenvectors and two eigenvalues $\frac{1\pm1/\pi}{4}$. The noise distribution of the stream $\{h_n h_n^\top \}_{n\geq0}$ uses a more involved construction. Consider $\alpha_n\sim\mathcal N(-\pi/2,\pi^2/4)$ and $\beta_n\sim\mathcal N(\pi/4,\pi^2/16)$ and $\tau_n\sim \mathcal{B}(1/2)$. Then define $\theta_n$ as:
\[
\theta_n= \tau_n \alpha_n+(1-\tau_n)\beta_n,
\]
and the stream $h_n$ to be
\[
h_n= \Big[\frac{\cos(\theta_n)}{\sqrt{(1-1/\pi)/2}}, \frac{\sin(\theta_n)}{\sqrt{(1+1/\pi)/2}}\Big].
\]
\myfig{counter}, compares the performance of averaged SGD with constant step size $\gamma=1$ and the decreasing step size $\gamma_n=\frac{1}{\sqrt{n}}$. We see that with constant step size both SGD and averaged SGD do not converge to the true solution. SGD oscillates around the solution in ball of radius $\sim \gamma$ and averaged SGD does converge but not to the correct solution (although is still contained in a ball of radius $\sim 10^{-4}$ around the correct solution). On the other hand, SGD with decreasing step size behaves just as well as with a Gaussian data stream. SGD converges to the solution at the slow rate $O(1/\sqrt{n})$, while averaged SGD converges at the fast rate $O(1/{n})$.

This interesting example shows that constant step size averaged SGD does not converge to the correct solution in all situations. However, it remains a open problem to investigate the convergence properties of constant step size SGD in the Gaussian case.

\end{document}